\documentclass[twoside]{article}

\usepackage[accepted]{aistats2024}
\usepackage{bbm}
\usepackage{amssymb, amsthm}
\usepackage{mathtools}
\usepackage{dsfont}
\usepackage{algorithm}
\usepackage{algpseudocode}
\usepackage{booktabs}
\usepackage{amssymb}
\usepackage{xcolor}
\usepackage{bigints}
\usepackage[colorlinks, linkcolor=purple, citecolor=blue, urlcolor=purple]{hyperref} 

\usepackage[round]{natbib}

\newtheorem{theorem}{Theorem}

\newtheorem{lemma}[theorem]{Lemma}
\newtheorem{corollary}[theorem]{Corollary}

\begin{document}

%

%

\newcommand{\objective}{\mathcal{H}}
\newcommand{\objectiver}{\mathcal{H}_r}
\newcommand{\vf}{\mathcal{V}}
\newcommand{\metric}{\mathcal{G}}
\newcommand{\leaves}{\mathcal{L}}
\newcommand{\node}{\mathcal{N}}
\newcommand{\splits}{\mathcal{S}}
\newcommand{\bound}{\mathcal{B}}
\newcommand{\prob}{\mathbb{P}}
\newcommand{\expec}{\mathbb{E}}
\newcommand{\ind}{\mathds{1}}
\newcommand{\var}{\mathbb{V}}
\newcommand{\children}{\textrm{Ch}}
\newcommand{\opt}{T^*}
\newcommand{\regret}{\overline{R}}
\newcommand{\longmethod}{\textsc{Thompson Sampling Decision Trees}}
\newcommand{\method}{\textsc{TSDT}}
\newcommand{\longosdt}{\textsc{Optimal Sparse Decision Trees}}
\newcommand{\osdt}{\textsc{OSDT}}
\newcommand{\otmane}[1]{\textcolor{red}{#1}}

\twocolumn[

\aistatstitle{Online Learning of Decision Trees with Thompson Sampling}

\aistatsauthor{ Ayman Chaouki \And Jesse Read \And  Albert Bifet }

\aistatsaddress{ LIX, Ecole Polytechnique, IP Paris \\AI Institute, University of Waikato \And  LIX, Ecole Polytechnique, \\ IP Paris\And AI Institute, University of Waikato \\ LTCI, T\'el\'ecom Paris, IP Paris} ]

\begin{abstract}
  Decision Trees are prominent prediction models for interpretable Machine Learning. They have been thoroughly researched, mostly in the batch setting with a fixed labelled dataset, leading to popular algorithms such as C4.5, ID3 and CART. Unfortunately, these methods are of heuristic nature, they rely on greedy splits offering no guarantees of global optimality and often leading to unnecessarily complex and hard-to-interpret Decision Trees. Recent breakthroughs addressed this suboptimality issue in the batch setting, but no such work has considered the online setting with data arriving in a stream. To this end, we devise a new Monte Carlo Tree Search algorithm, \longmethod~(\method), able to produce optimal Decision Trees in an online setting. We analyse our algorithm and prove its almost sure convergence to the optimal tree. Furthermore, we conduct extensive experiments to validate our findings empirically. The proposed \method~outperforms existing algorithms on several benchmarks, all while presenting the practical advantage of being tailored to the online setting.
\end{abstract}

\section{INTRODUCTION}

Interpretable Machine Learning is crucial in sensitive domains, like medicine, where high-stakes decisions have to be justified. Due to their extraction of simple decision rules, Decision Trees (DTs) are very popular in this context. Unfortunately, finding the optimal DT is NP-complete \citep{laurent1976constructing}, and for this reason, popular batch algorithms greedily construct a DT by splitting its leaves according to some local gain metric, such approach is used in ID3 \citep{quinlan1986induction}, C4.5 \citep{quinlan2014c4} and CART \citep{breiman1984classification} to name a few. Due to this heuristic nature, these approaches offer no optimality guarantees. In fact, they often lead to suboptimal DTs that are unnecessarily complex and hard to interpret, contradicting the main motivation behind DTs. 

In many modern applications, data is supplied through a stream rather than a fixed data set, this renders most batch algorithms obsolete, which led to the emergence of the data stream (or online) learning paradigm \citep{bifet2009data}. The classic batch DT algorithms are ill-suited for online learning since they calculate a splitting gain metric on a whole data set. In response, \cite{domingos2000mining} introduced the VFDT algorithm, which constructs DTs in an online fashion. VFDT estimates the gain of each split using a statistical test based on Hoeffding's inequality. This approach yielded a principled algorithm and laid the foundation for subsequent developments \citep{hulten2001mining, bifet2009adaptive, manapragada2018extremely}, with advances primarily focusing on improving the quality of the statistical tests \citep{jin2003efficient, rutkowski2012decision, rutkowski2013decision}. Much like their batch counterparts, these online methods are heuristic in nature, and consequently, they are susceptible to the suboptimality issue.

In this work, we propose a method that circumvents these limitations, yielding an online algorithm proven to converge to the optimal DT. We consider online classification problems with categorical attributes, for which we seek the optimal DT balancing between the accuracy and the number of splits. To achieve this, we frame the problem as a Markov Decision Process (MDP) where the optimal policy leads to the optimal DT. We solve this MDP with a novel Monte Carlo Tree Search (MCTS) algorithm that we call \longmethod~(\method). \method~employs a Thompson Sampling policy that converges almost surely to the optimal policy. In our experiments, we highlight the limitations of traditional greedy online DT methods, such as VFDT and EFDT \citep{manapragada2018extremely}, and demonstrate how \method~effectively circumvents these shortcomings. Due to the lack of literature on optimal online DTs, we compare \method~with recent successful batch optimal DT algorithms by feeding benchmark datasets to \method~as streams, \method~clearly outperforms DL8.5 \citep{aglin2020learning} and matches or surpasses the performance of \osdt~\citep{hu2019optimal}.

\section{RELATED WORK}

In the batch setting, the suboptimality issue of DTs has been the subject of multiple research papers focused mainly on mathematical programming, see \citep{bennett1994global, bennett1996optimal, norouzi2015efficient, bertsimas2017optimal, verwer2019learning}. These methods optimise internal splits within a fixed DT structure, making the problem more manageable but potentially missing the optimal DT. Recently, branch and bound methods were proposed to mitigate this issue and yielded the DL8.5 algorithm \citep{aglin2020learning} and \osdt~\citep{hu2019optimal} among others. A subsequent algorithm, GOSDT \citep{lin2020generalized}, generalises \osdt~to other objective functions including F-score, AUC and partial area under the ROC convex hull. However, these methods are limited to binary attributes, necessitating a preliminary binary encoding of the data. Moreover, the choice of this binary encoding may significantly influence the complexity of the solution, as demonstrated in our experiments. All the aforementioned methods operate solely in the batch learning paradigm, lacking a straightforward extension to the online setting.

The closest work to ours is perhaps \citep{nunes2018monte} since the authors use MCTS, see \citep{browne2012survey} for a survey about MCTS. \cite{nunes2018monte} define a rollout policy that completes the selected DT with C4.5 on an induction set, then it estimates the value of the selected DT by evaluating its performance on a validation set. This approach does not differentiate between DTs of different complexities, in fact, the authors rely on a custom definition of terminal states in terms of predefined maximum depth and number of instances, alongside C4.5's pruning strategy. Additionally, by virtue of using C4.5, a pure batch algorithm, this algorithm is not applicable for data streams. 

Our proposed method is a Value Iteration approach \citep{sutton2018reinforcement} that uses Thompson Sampling policy within a MCTS framework. Unlike Temporal Difference methods, such as Q-Learning and SARSA, general convergence results for Monte Carlo methods remain an open theoretical question, noted in \citep[p.~99 and p.~103]{sutton2018reinforcement} as "one of the most fundamental open theoretical questions". Some convergence results were established under specific assumptions. \cite{wang2020convergence} prove almost sure optimal convergence of the policy for Monte Carlo with Exploring Starts, while \cite{dong2022convergence} show a similar result for Monte Carlo UCB. These results pertain to MDPs with finite random length episodes where the optimal policy does not revisit states. In our case, although our MDP features similar properties, the rewards are unknown and merely estimated. We investigate the convergence properties of MCTS with Thompson Sampling policy within our specific MDP. To the best of our knowledge, no prior work has carried such analysis under this assumption. The closest related work, found in \citep{bai2013bayesian}, only considers discounted MDPs with finite fixed horizons, and does not provide a formal convergence proof, see \citep[Section 3.5]{bai2013bayesian}.

\section{PROBLEM FORMULATION}

Let $X = \left( X^{\left( 1\right)}, \ldots, X^{\left( q\right)}\right)$ be the input with categorical attributes and $Y \in \{ 1, \ldots, K\}$ the class to predict. Data samples $\left( X_i, Y_i\right)$ arrive incrementally through a stream, they are i.i.d. and follow a joint probability distribution $P_{X, Y}$. Let $T$ be a DT and $\leaves\left( T\right)$ the set of leaves of $T$, for each leaf $l \in \leaves\left( T\right)$ and class $k \in \{ 1, \ldots, K\}$, let $p\left( l\right) = \prob\left[ X \in l\right]$ denote the probability of event "The subset of the input space, described by leaf $l$, contains $X$" and $p_k\left( l\right) = \prob\left[ Y=k | X \in l\right]$ the probability that $Y=k$ given $X \in l$. For any input $X$ we also denote $l\left( X\right)$ the leaf $l$ that contains $X$, i.e. the leaf $l$ such that $X \in l$. Let $\objective\left( T\right) = \prob\left[ T\left( X\right) = Y\right]$ be the accuracy of $T$ where $T\left( X\right) = T\left( l\left( X\right)\right) = \textrm{Argmax}_k\{p_k\left( l\left( X\right)\right)\}$ is the predicted class of $X$ according to $T$. If we define our objective to maximise as $\objective\left( T\right)$, then the full DT that exhaustively employs all the possible splits is a trivial solution. However, this solution is an uninterpretable DT of maximum depth that just classifies the inputs point-wise, as such, it is of no interest. We seek to balance between maximising the accuracy and minimising the complexity, the latter condition is for interpretation purposes. To this end, we introduce the regularised objective:
$$\objectiver\left( T\right) = \prob\left[ T\left( X\right) = Y\right] - \lambda\splits\left( T\right)$$
Where $\lambda \ge 0$ is a penalty parameter and $\splits\left( T\right)$ is the number of splits in $T$. We note that \cite{hu2019optimal} introduce a similar objective, but the authors penalise the number of leaves $|\leaves\left( T\right)|$ rather than the number of splits $\splits\left( T\right)$. Our choice is motivated by the MDP we define in the following section.

\subsection{Markov Decision Process (MDP)}

We introduce our undiscounted episodic MDP with finite random length episodes as follows:\\
\textbf{State:} Our state space is the space of DTs. \\
\textbf{Action:} There are two types of actions, split actions split a leaf with respect to an attribute and the terminal action ends the episode, in which case we transition from the current state $T$ to the terminal state denoted $\overline{T}$ which represents the same DT as $T$.\\
\textbf{Transition Dynamics:} When taking an action, we transition from state $T$ to state $T'$ deterministically, in which case we denote the transition $T \rightarrow T'$. The set of next states from $T$ is denoted $\children\left( T\right)$, which signifies the children of $T$.\\
\textbf{Reward:} In any non-terminal state $T$, split actions yield reward $-\lambda$ and the terminal action yields reward $\objective\left( T\right) = \prob\left[ T\left( X\right) = Y\right]$, \emph{which is unknown}. $r\left( T, T'\right)$ is the reward of the transition $T \rightarrow T'$.

In the next paragraph, we link the search for the optimal policy to that of the optimal DT.\\
A stochastic policy $\pi$ maps each non-terminal state $T$ to a distribution over $\children\left( T\right)$, for any $T' \in \children\left( T\right)$ we denote $\pi\left( T' | T\right)$ the probability of the transition $T \rightarrow T'$ according to $\pi$. If $\pi$ is deterministic, $\pi\left( T\right)$ is the next state from $T$ according to $\pi$. Let $T = T^{\left( 0\right)} \overset{\pi}{\rightarrow} T^{\left( 1\right)} \overset{\pi}{\rightarrow} \ldots, \overset{\pi}{\rightarrow} T^{\left( N\right)}$ be the episode that stems from following policy $\pi$ starting from state $T$; $T^{\left( N\right)} = \overline{T^{\left( N-1\right)}}$ is terminal. The value of $\pi$ at $T$ is defined as:
$$
\vf^{\pi}\left( T\right) = \expec\left[ \sum_{j=1}^N r\left( T^{\left( j-1\right)}, T^{\left( j\right)}\right)\right]
$$
For convenience, we also define, for all terminal states $\overline{T}$, $\vf^\pi\left( \overline{T}\right) = \objective\left( \overline{T}\right) = \objective\left( T\right)$. If $\pi$ is deterministic, then we get:
\begin{align*}
    \vf^\pi\left( T\right) = \lambda\splits\left( T\right) + \objectiver\left( T^{\left( N-1\right)}\right)
\end{align*}
Let $R$ denote the root state (DT with only one leaf), we have $\splits\left( R\right) = 0$ and therefore:
\begin{equation*}
    \vf^\pi\left( R\right) = \lambda\splits\left( R\right) + \objectiver\left( T^{\left( N-1\right)}\right) = \objectiver\left( T^{\left( N-1\right)}\right)
\end{equation*}
Let $\pi^* \in \textrm{Argmax}_\pi\vf^\pi\left( R\right)$, the optimal policy $\pi^*$ exists and is deterministic because our MDP is finite. Let $R \overset{\pi^*}{\rightarrow} \ldots, \overset{\pi^*}{\rightarrow} T^* \overset{\pi^*}{\rightarrow} \overline{T^*}$, then $\objectiver\left( T^*\right) = \vf^{\pi^*}\left( R\right) \ge \vf^{\pi}\left( R\right)$ for any policy $\pi$. On the other hand, any DT $T$ is constructed from a series of splits of the root $R$, thus there always exists a policy $\pi$ such that $R \overset{\pi}{\rightarrow} \ldots, \overset{\pi}{\rightarrow} T \overset{\pi}{\rightarrow} \overline{T}$, and consequently $\vf^\pi\left( R\right) = \objective_r\left( T\right)$. As a result, $\objectiver\left( T^*\right) \ge \objectiver\left( T\right)$ for any DT $T$, establishing the optimality of $T^*$. We can find $T^*$ by deriving $\pi^*$ first and then following it.

\subsection{Tree representation of the State-Action Space}

As is custom in MCTS, the State-Action space is represented as a Tree called the Search Tree. \textbf{We refer to the nodes of the Search Tree as Search Nodes to avoid confusion with the nodes of DTs}. These Search Nodes serve as representations of states, which are DTs in the context of our MDP. The edges within the Search Tree correspond to actions. Throughout this work, we refer to DTs, states and Search Nodes interchangeably. The root of the Search Tree is the initial state $R$ (the root DT), and its leaves, called Search Leaves, are the terminal states. Figure \ref{fig:search-tree} depicts a segment of the Search Tree.

\begin{figure}[t]
\centering
\includegraphics[width=0.35\textwidth]{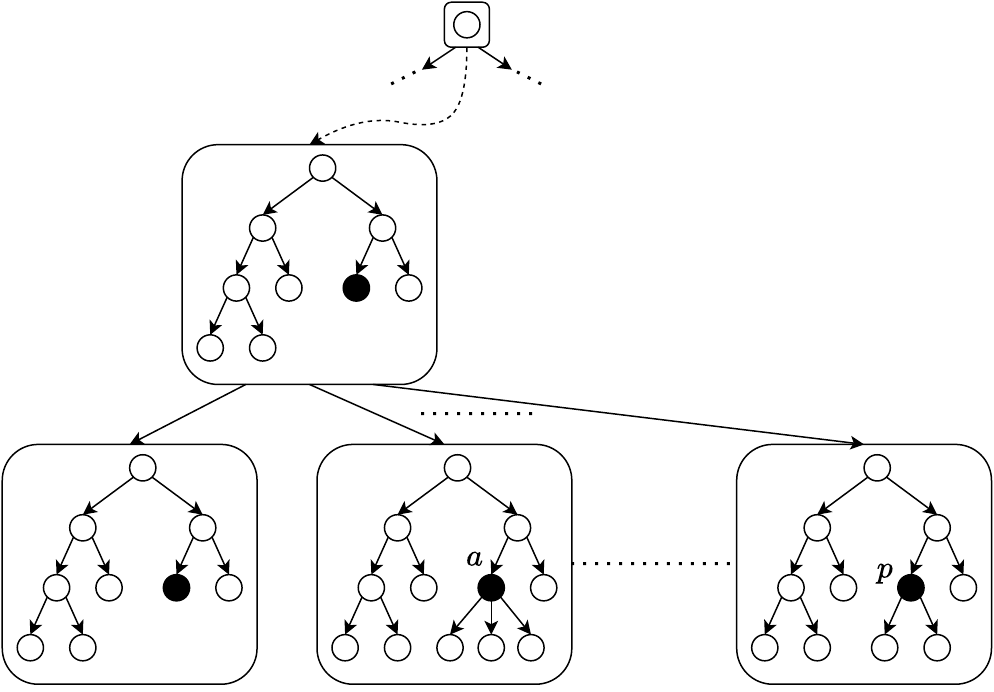}
\caption{Each Search Node is a state and each edge an action. The left-most edge is the terminal action, hence why both the parent and child Search Nodes represent the same DT. The remaining edges are split actions with respect to the black leaf.}
\label{fig:search-tree}
\end{figure}

\section{\method}

In this section, we introduce our method. If we know $\vf^{\pi^*}\left( \overline{T}\right) = \objective\left( T\right)$ for all Search Leaves $\overline{T}$, then we can Backpropagate these values up the Search Tree and recursively deduce $\vf^{\pi^*}\left( T\right)$ for all internal Search Nodes with the Bellman Optimality Equation:
\begin{equation}
    \vf^{\pi^*}\left( T\right) = \max_{T' \in \children\left( T\right)}\Big\{ -\lambda\ind\Big\{ T' \neq \overline{T}\Big\} + \vf^{\pi^*}\left( T'\right) \Big\} \label{eq:bellman-value}
\end{equation}
Unfortunately, the values $\vf^{\pi^*}\left( \overline{T}\right)$ are unknown, this prompts us to estimate them, but which Search Leaves should we prioritise? We need a policy with an efficient Exploration-Exploitation trade-off. Several notable options can be considered, among which UCB \citep{auer2002finite}, that was popularised, in the context of MCTS, by the UCT algorithm \citep{kocsis2006bandit}. UCB is out of the scope of this paper, we analyse it in one of our ongoing works. We consider Thompson Sampling instead. To use this policy, we need to estimate the values within a Bayesian framework.

\subsection{Estimating $\vf^{\pi^*}\left( \overline{T}\right)$ for a Search Leaf $\overline{T}$}

\label{sec:search-leaf}

For any Search Leaf $\overline{T}$, we have:
\begin{align}
    \vf^{\pi^*}\left( \overline{T}\right) &= \objective\left( T\right) = \prob\left[ T\left( X\right) = Y\right] \nonumber\\
    &= \sum_{l \in \leaves\left( T\right)}p\left( l\right)\expec\left[ \ind\{ T\left( X\right) = Y\}|X \in l\right] \label{eq:value-search-leaf}
\end{align}
From Equation \eqref{eq:value-search-leaf}, given observed data $\{ \left( X_i, Y_i\right)\}_{i=1}^N$ in $T$, we define the posterior on $\vf^{\pi^*}\left( \overline{T}\right)$ with:
\begin{equation}
    \theta_{\overline{T}} = \sum_{l \in \leaves\left( T\right)}\hat{p}\left( l\right)\theta_{T, l}
    \label{eq:prior-search-leaf}
\end{equation}
Where $\hat{p}\left( l\right)$ is an estimator of $p\left( l\right)$ and $\theta_{T, l}$ follows a posterior distribution on $\expec\left[ \ind\{ T\left( X\right) = Y\}|X \in l\right]$ given $\{ \left( X_i, Y_i\right)\}_{i=1}^N$. Since $\expec\left[ \ind\{ T\left( X\right) = Y\}|X \in l\right]$ is the mean of a Bernoulli variable, we are tempted to use its Beta conjugate prior with the following updates:
\begin{align*}
    \theta_{T, l} &\sim \textrm{Beta}\left( \alpha_{T, l}, \beta_{T, l}\right)\\
    \alpha_{T, l} &= 1 + \sum_{i=1}^N\ind\{ X_i \in l\}\ind\{ T\left( X_i\right) = Y_i\}; \\
    \beta_{T, l} &= 1 + \sum_{i=1}^N \ind\{ X_i \in l\} - \sum_{i=1}^N\ind\{ X_i \in l\}\ind\{ T\left( X_i\right) = Y_i\}
\end{align*}
However, the challenge pertains to the unknown nature of $T\left( X_i\right) = \textrm{Argmax}_{k}\{ p_k\left( l\left( X_i\right)\right)\}$. We solve this issue with the empirical average estimates:
\begin{align}
    \forall{l \in \leaves\left( T\right)}: \hat{p}_k^{\left( i\right)}\left( l\right) &= \frac{\sum_{j=1}^i\ind\{ X_j \in l\}\ind\{ Y_j = k\}}{\sum_{j=1}^i \ind\{ X_j \in l\}} \nonumber \\
    \hat{T}_i\left( l\right) &= \textrm{Argmax}_{k}\{ \hat{p}_k^{\left( i\right)}\left( l\right)\} \label{eq:estimate-predictor}
\end{align}
Then we rather update $\alpha_{T, l}$ and $\beta_{T, l}$ with:
\begin{align}
    \alpha_{T, l} &= 1 + \sum_{i=2}^N\ind\{ X_i \in l\}\ind\{ \hat{T}_{i-1}\left( X_i\right) = Y_i\}; \label{eq:alpha} \\
    \beta_{T, l} &= 1 + \sum_{i=2}^N \ind\{ X_i \in l\} \nonumber \\
    &- \sum_{i=2}^N\ind\{ X_i \in l\}\ind\{ \hat{T}_{i-1}\left( X_i\right) = Y_i\} \label{eq:beta}
\end{align}

Now $\theta_{\overline{T}}$ (Equation \eqref{eq:prior-search-leaf}) is a linear combination of the Beta variables $\theta_{T, l}$, and its distribution is not easy to infer. For this reason, we consider the common Normal approximation of the Beta distribution where we match the first two moments.
\begin{align}
    \theta_{T, l} &\sim \mathcal{N}\left( \mu_{T, l}, \left( \sigma_{T, l}\right)^2\right); \; \mu_{T, l} = \frac{\alpha_{T, l}}{\alpha_{T, l} + \beta_{T, l}} \label{eq:prior-search-leaf-leaf-1}\\
    \left( \sigma_{T, l}\right)^2 &= \frac{\alpha_{T, l}\beta_{T, l}}{\left( \alpha_{T, l} + \beta_{T, l}\right)^2\left( 1 + \alpha_{T, l} + \beta_{T, l}\right)} \label{eq:prior-search-leaf-leaf-2}
\end{align}

This makes $\theta_{\overline{T}} = \sum_{l \in \leaves\left( T\right)}\hat{p}\left( l\right)\theta_{T, l}$ a linear combination of Normal random variables, and therefore:
\begin{align}
    \theta_{\overline{T}} &\sim \mathcal{N}\left( \mu_{\overline{T}}, \left( \sigma_{\overline{T}}\right)^2\right); \; \mu_{\overline{T}} = \sum_{l \in \leaves\left( T\right)} \hat{p}\left( l\right)\mu_{T, l} \label{eq:prior-search-leaf-1}\\
    \left( \sigma_{\overline{T}}\right)^2 &= \sum_{l \in \leaves\left( T\right)} \hat{p}\left( l\right)^2 \left( \sigma_{T, l}\right)^2 \label{eq:prior-search-leaf-2}
\end{align}
We recall that this posterior distribution is conditioned on the observed data $\{ \left( X_i, Y_i\right)\}_{i=1}^N$ in $T$. To finalise the definition of $\theta_{\overline{T}}$, we still need to define $\hat{p}\left( l\right)$. We defer this task to Section \ref{sec:p(l)} as we are currently lacking some key insights from the Algorithm. For the time being, we assume having such estimator that is completely defined by $\{ X_i\}_{i=1}^N$ and consistent $\hat{p}\left( l\right) \xrightarrow[N \rightarrow \infty]{\textrm{a.s}} p\left( l\right)$.

\subsection{Estimating $\vf^{\pi^*}\left( T\right)$ for an internal Search Node $T$}

\label{sec:search-node}

Let $T$ be an internal Search Node, which is a non-terminal state. Value Iteration updates the estimate of $\vf^{\pi^*}\left( T\right)$ according to the Bellman Optimality Equation \eqref{eq:bellman-value}. Thus, we define the posterior on $\vf^{\pi^*}\left( T\right)$ given all the observed data $\{ \left( X_i, Y_i\right)\}_{i=1}^N$ in $T$ as:
\begin{equation}
    \theta_T = \max_{T' \in \children\left( T\right)}\Big\{ -\lambda\ind\Big\{ T' \neq \overline{T}\Big\} + \theta_{T'} \Big\} \label{eq:theta-internal}
\end{equation}
\textbf{What is the posterior distribution of $\theta_T$ given $\{ \left( X_i, Y_i\right)\}_{i=1}^N$?}

We suppose $\forall{T'} \in \children\left( T\right): \theta_{T'} \sim \mathcal{N}\left( \mu_{T'}, \left( \sigma_{T'}\right)^2\right)$. This is motivated by an inductive reasoning, indeed, if we show that $\theta_T$ is Normally distributed, then since the posteriors of all Search Leaves are Normal, as defined in Section \ref{sec:search-leaf}, we would recursively infer that the posteriors of all internal Search Nodes are Normal.\\
Unfortunately, the maximum of Normal variables, as defined in Equation \eqref{eq:theta-internal}, is not Normal, this observation is documented in \citep{clark1961greatest, sinha2007advances}. To solve this issue, a first approach is to use a Normal approximation of the distribution of the maximum in Equation \eqref{eq:theta-internal}, this is achieved by recursively applying Clark's mean and variance formula for the maximum of two Normal variables as demonstrated in \cite[Section 4.1]{10.5555/3023549.3023618}. This leads to our first version of \method, we present the details of this Backpropagation scheme in Appendix \ref{appendix:backpropagation}. Nevertheless, this approximation incurs a substantial computational cost as the number of children $|\children\left( T\right)|$ increases. Moreover, as outlined by \cite{sinha2007advances}, the order in which the recursive approximations are executed may significantly affect the quality of the overall approximation. For these reasons, we introduce an alternative, more straightforward approach to Backpropagating the posterior distributions $\theta_{T'}$ to $\theta_T$. Specifically, we assign to $\theta_T$ the posterior distribution of the child with maximum posterior mean:
\begin{equation}
    \widetilde{T} = \textrm{Argmax}_{T' \in \children\left( T\right)}\{ \mu_{T'}\}; \; \theta_T \sim \mathcal{N}\left( \mu_{\widetilde{T}}, \left( \sigma_{\widetilde{T}}\right)^2\right) \label{eq:theta-fast-tsdt}
\end{equation}
We call this second version of the Algorithm Fast-\method, it is more practical than the first one and exhibits better computational efficiency. While in general, the distribution in formula \eqref{eq:theta-fast-tsdt} may not serve as an accurate approximation for the distribution of the maximum in Equation \eqref{eq:theta-internal}, it progressively improves as $\theta_{T'}$, for children $T'$, concentrate around their means. This concentration occurs as more data is accumulated within $T'$, as exemplified in Equations \eqref{eq:prior-search-leaf-2} and \eqref{eq:prior-search-leaf-leaf-2} for Search Leaves. Furthermore, it is worth noting that Fast-\method~displays a substantial gain in computational efficiency and also (surprisingly!) in performance compared to \method.

\subsection{The Algorithm}

\label{sec:alg}

Algorithm \ref{alg:TSDT} is an abstract description of \method~and Fast-\method. In the following, we view it from the perspective of an incremental construction of the Search Tree representation. Initially, our Search Tree representation contains the root $R$ only, then at each iteration $t$, we follow the steps illustrated in Figure \ref{fig:mcts-tsdt}.
\begin{itemize}
    \item \textbf{Selection:} Line 4. Starting from $R$, descend the Search Tree by choosing a child according to the current policy $\pi_t$ until reaching a Search Leaf.
    \item \textbf{Simulation:} Line 5. We observe new $m$ incoming data from the stream in $T^{\left( N-1\right)}$. The objective is to use the accumulated observed data in $T^{\left( N-1\right)}$ to, either initialise the posteriors of children $T' \in \children\left( T^{\left( N-1\right)}\right)$ with $\theta_{T'} = \theta_{\overline{T'}}$ (Line 11), or to just update the posterior of $\theta_{T^{\left( N\right)}}$ (Line 14).
    These posteriors will in turn update the posterior of $\theta_T$, as per Section \ref{sec:search-node}, during Backpropagation.
    \item \textbf{Expansion:} Line 7. If $T^{\left( N-1\right)}$ is visited for the first time, we add its children Search Nodes $T' \in \children\left( T^{\left( N-1\right)}\right)$ to our Search Tree representation.
    \item \textbf{Backpropagation:} Loop from Line 16 to 21. Recursively update the posterior of the ancestors $\theta_{T^{\left( j\right)}}$ (Line 17), for $j=N-1$ to $0$, with the internal Search Nodes posterior updates as per Section \ref{sec:search-node} (formula \eqref{eq:theta-fast-tsdt} for Fast-\method, and the recursive Normal approximation of the maximum for \method). At the same time, update the Thompson Sampling policy (Loop from Line 18 to 20). 
\end{itemize}
In line 23, after $M$ iterations of \method, we define the greedy policy $\pi$ with respect to the posterior means. Then we unroll $\pi$: $R = T^{\left( 0\right)} \overset{\pi}{\rightarrow} T^{\left( 1\right)} \overset{\pi}{\rightarrow} \ldots, \overset{\pi}{\rightarrow} T^{\left( N\right)}$ and return the proposed solution $T^{\left( N\right)}$ (Lines 24, 25).

To complete the description of Algorithm \ref{alg:TSDT}, we still need to answer the following question: \textbf{How does the Simulation step allow us to initialise the posteriors for the children?} To answer this question, we introduce some new statistics.\\
Let $T$ be a Search Node that was simulated, and suppose we observe data $\{ \left( X_s, Y_s\right)\}_{s=1}^N$ in $T$. Let $l \in \leaves\left( T\right)$ be a leaf of $T$ and $T^{\left( l, i\right)} \in \children\left( T\right)$ the child that stems from splitting $l$ with respect to attribute $X^{\left( i\right)}$. Our objective here is to use data $\{ \left( X_s, Y_s\right)\}_{s=1}^N$ to initialise $\theta_{T^{\left( l, i\right)}} = \theta_{\overline{T^{\left( l, i\right)}}}$. According to Equations \eqref{eq:prior-search-leaf-leaf-1}, \eqref{eq:prior-search-leaf-leaf-2}, \eqref{eq:prior-search-leaf-1} and \eqref{eq:prior-search-leaf-2}, this is achieved by calculating $\alpha_{T^{\left( l, i\right)}, l'}, \beta_{T^{\left( l, i\right)}, l'}$ for all leaves $l' \in \leaves\left( T^{\left( l, i\right)}\right)$ (and also $\hat{p}\left( l'\right)$ which, we remind the reader, is deferred to Section \ref{sec:p(l)}). Let $l' \in \leaves\left( T^{\left( l, i\right)}\right)$, if $l'$ is not a child of $l$, then $l'$ is a common leaf between $T^{\left( l, i\right)}$ and $T$, thus $\alpha_{T^{\left( l, i\right)}, l'} = \alpha_{T, l'}$ and $\beta_{T^{\left( l, i\right)}, l'} = \beta_{T, l'}$, these are straightforwardly calculated with the observed data $\{ \left( X_i, Y_i\right)\}_{i=1}^N$ in $T$ using Equations \eqref{eq:estimate-predictor}, \eqref{eq:alpha} and \eqref{eq:beta}. If $l'$ is a child of $l$, then there exists $j$ such that $l' = l_{ij}$ is the child of $l$ that corresponds to attribute $X^{\left( i\right)}$ being equal to $j$. For $N' \le N$, we define:
\begin{align*}
    n_{ijk}\left( N', l\right) &= \sum_{s=1}^{N'}\ind\{ Y_s = k\}\ind\{ X_s^{(i)} = j\}\ind\{ X_s \in l\}
\end{align*}
On the subset $\{ \left( X_s, Y_s\right)\}_{s=1}^{N'}$, $n_{ijk}\left( N', l\right)$ is the number of inputs $X_s \in l$ of class $k$ satisfying $X_s^{\left( i\right)} = j$.
Then we can track the estimates:
\begin{align*}
    &\hat{T}_{N'}^{\left( l, i\right)}\left( l_{ij}\right) = \textrm{Argmax}_k \Bigg\{ \frac{n_{ijk}\left( N', l\right)}{\sum_k n_{ijk}\left( N', l\right)} \Bigg\} \\
    &m\left( l_{ij}\right) = \sum_{s=2}^N\ind\{ \hat{T}_{s-1}^{\left( l, i\right)}\left( X_s\right) = Y_s\}\ind\{ X_s^{(i)} = j\}\ind\{ X_s \in l\}\\
    &\alpha_{T^{\left( l, i\right)}, l_{ij}} = 1 + m\left( l_{ij}\right); \\
    &\beta_{T^{\left( l, i\right)}, l_{ij}} = 1 + \sum_{ijk}n_{ijk}\left( N, l\right) - m\left( l_{ij}\right)
\end{align*}
With this, the initialisation $\theta_{T^{\left( l, i\right)}} = \theta_{\overline{T^{\left( l, i\right)}}}$ is now complete. In summary, the introduced $n_{ijk}\left( N', l\right)$ statistics at the leaves $l \in \leaves\left( T\right)$ allow us to use the accumulated observed data $\{ \left( X_i, Y_i\right)\}_{i=1}^N$ in $T$ to initialise the posterior $\theta_{T'} = \theta_{\overline{T'}}$ for all children $T' \in \children\left( T\right)$.

\begin{algorithm}[tb]
\caption{\method, Fast-\method}\label{alg:TSDT}
\begin{algorithmic}[1]

\State \textbf{Input:} $M$ number of iterations, $m$ number of observed samples per Simulation, $\lambda \ge 0$.
\State Initialise $\pi_0\left( \overline{T} \Big| T\right) = 1$ and \emph{fully\_expanded($T$) = False} for all non-terminal states $T$

\For{$t=1$ to $M$}
    \State Unroll $\pi_t$: $R = T^{\left( 0\right)} \overset{\pi_t}{\rightarrow} T^{\left( 1\right)} \overset{\pi_t}{\rightarrow} \ldots, \overset{\pi_t}{\rightarrow} T^{\left( N\right)}$
    \State \emph{Simulate($T^{\left( N-1\right)}$)}
    \If{not \emph{fully\_expanded($T^{\left( N-1\right)}$)}}
        \State \emph{Expand($T^{\left( N-1\right)}$)}
        \State \emph{fully\_expanded($T^{\left( N-1\right)}$) = True}
        \For{$T' \in \children\left( T^{\left( N-1\right)}\right)$} 
            \State Update the posterior of $\theta_{\overline{T'}}$
            \State Initialise $\theta_{T'} = \theta_{\overline{T'}}$
        \EndFor
    \Else
        \State Update the posterior of $\theta_{T^{\left( N\right)}}$
    \EndIf
    \For{$j=N-1$ to $0$}
        \State Update the posterior of $\theta_{T^{\left( j\right)}}$
        \For{$T' \in \children\left( T^{\left( j\right)}\right)$}
            \State Update the policy at $T^{\left( j\right)} \rightarrow T'$
            $$\pi_{t+1}\left( T' \Big| T^{\left( j\right)}\right) = \prob\left[ \theta_{T'} = \max_{T" \in \children\left( T^{\left( j\right)}\right)}\{ \theta_{T"}\}\right]$$
        \EndFor

    \EndFor
\EndFor
\State Define $\pi\left( T\right) = \textrm{Argmax}_{T' \in \children\left( T\right)}\Big\{ -\lambda\ind\{ T' \neq \overline{T}\} + \mu_{T'} \Big\}$ for all non-terminal states $T$
\State Unroll $\pi$: $R = T^{\left( 0\right)} \overset{\pi}{\rightarrow} T^{\left( 1\right)} \overset{\pi}{\rightarrow} \ldots, \overset{\pi}{\rightarrow} T^{\left( N\right)}$
\State \Return $T^{\left( N\right)}$
\end{algorithmic}
\end{algorithm}

In the following, we provide the optimal convergence result for both \method~and Fast-\method.
\begin{theorem}
    \label{thm:main}
    Let time $t$ denote the number of iterations of \method~and Fast-\method, then any Search Node $T$ satisfies the following:
    $$
    \mu_T \xlongrightarrow[t \rightarrow \infty]{\textrm{a.s}} \vf^{\pi^*}\left( T\right), \left( \sigma_T\right)^2 \xlongrightarrow[t \rightarrow \infty]{\textrm{a.s}} 0
    $$
    and any internal Search Node $T$ satisfies: 
    $$\pi_t\left( \pi^*\left( T\right) \Big| T\right) \xlongrightarrow[t \rightarrow \infty]{\textrm{a.s}} 1$$
\end{theorem}
Theorem \ref{thm:main} states the concentration of $\theta_T$ around the optimal value $\vf^{\pi^*}\left( T\right)$ for any Search Node $T$. Additionally, it asserts the concentration of the policy $\pi_t\left( .|T\right)$ around the optimal action $\pi^*\left( T\right)$ for any internal Search Node $T$. This Theorem provides an asymptotic guarantee of optimality, which is valuable. However, it would be ideal to have some finite time guarantees in the form of PAC-bounds or rates of convergence. Unfortunately, such guarantees are primarily derived for simpler settings like bandit problems. In the context of MDPs, most convergence guarantees are asymptotic, and the issue of finite-time guarantees remains open in many cases. We believe that the tail inequality we derive in Theorem \ref{thm:visits-proba} marks a first step towards achieving this goal in a future work. The idea is that by controlling the concentration of the posteriors at the level of Search Leaves, it may be possible to propagate this control up the Search Tree to the root state $R$. This would allow us to derive a time-dependent concentration probability of the posterior of $\theta_R$ around the true $\vf^{\pi^*}\left( R\right)$. However, this Backpropagation reasoning is a non-trivial challenge that warrants further exploration in future work.
\begin{theorem}
\label{thm:visits-proba}
Let $T$ be an internal Search Node with $\children\left( T\right) = \{ T_1, \ldots, T_w\}$ Search Leaves and $w \ge 2$. Let $t$ denote the number of visits of $T$ and $N_{T_j}\left( t\right)$ the number of visits of $T_j$ up to $t$. Define  $M_w = 1 + \sqrt{\frac{2}{\sqrt{3}}}\textrm{erfc}^{-1}\left( \frac{1}{w-1}\right)$, then $\forall{T_j  \in \children\left( T\right)}$:
\begin{align*}
    &\prob\left[ N_{T_j}\left( t\right)m \le \frac{\log t}{4|\leaves\left( T_j\right)|M_w}\right] \le \\
    &\exp\left[ -\frac{2}{t}\left( \frac{t^{3/4}}{\sqrt{\pi}\phi\left( t\right)} - \frac{\log t}{4m|\leaves\left( T_j\right)|M_w^2}\right)^2\right]
\end{align*}
Where $\phi\left( t\right) = \sqrt{\frac{\log t}{4}} + \sqrt{\frac{\log t}{4}+2}$
\end{theorem}

\begin{figure}
  \centering
  \includegraphics[width=.4\textwidth]{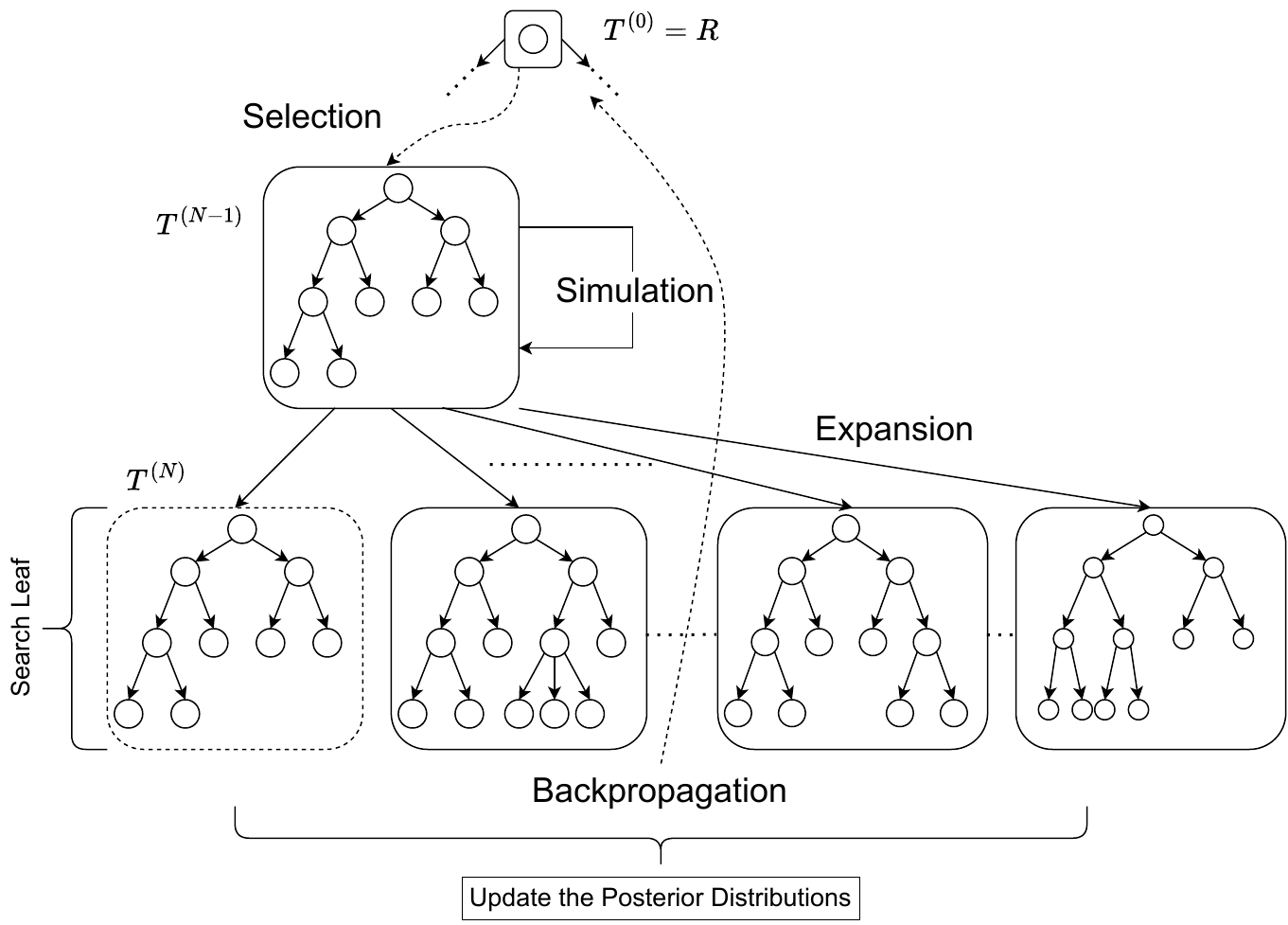}
  \caption{One iteration of \method. The Search Node in dashed lines is the Search Leaf $T^{\left( N\right)} = \overline{T^{\left( N-1\right)}}$.}
  \label{fig:mcts-tsdt}
\end{figure}

\subsection{Estimator $\hat{p}\left( l\right)$}

\label{sec:p(l)}

Let $T$ be a DT and $l \in \leaves\left( T\right)$ some leaf. We recall from Equations \eqref{eq:value-search-leaf} and \eqref{eq:prior-search-leaf} that we need an estimator $\hat{p}\left( l\right)$ of $p\left( l\right) = \prob\left[ X \in l\right]$. Suppose we observe data $\{ \left( X_i, Y_i\right)\}_{i=1}^N$ in $T$, the most straightforward estimator is the empirical average. However, Algorithm \ref{alg:TSDT} incorporates mechanisms that make it sample efficient. It does not copy Decision Tree nodes, thus the $n_{ijk}\left( N, \eta\right)$, of any node $\eta$, are updated whenever data $\left( X_s, Y_s\right)$ is observed (during Simulation) with $X_s \in \eta$, regardless of which Search Node was simulated. In addition, even though $n_{ijk}\left( N, l\right)$ can be used to calculate the empirical averages:
\begin{align*}
    \hat{p}\left( l\right) &= \frac{\sum_{ijk}n_{ijk}\left( N, l\right)}{\sum_{l \in \leaves\left( T\right)}\sum_{ijk}n_{ijk}\left( N, l\right)}\\
    \hat{p}\left( l_{ij}\right) &= \frac{\sum_{k}n_{ijk}\left( N, l\right)}{\sum_{l \in \leaves\left( T\right)}\sum_{ijk}n_{ijk}\left( N, \eta\right)}
\end{align*}
for any leaf $l \in \leaves\left( T\right)$ and any child node $l_{ij}$ of $l$. $n_{ijk}\left( N, l\right)$ cannot calculate the empirical average for the children of $l_{ij}$. This limited scope, along with not copying Decision Tree nodes, cause the observed marginal distribution of input $X$ to shift from the true marginal distribution. We explain this phenomenon in the next paragraph and illustrate it in Figure \ref{fig:degeneracy}.

When expanding $R$, the observed data in $r$ update the empirical averages for nodes $a$ and $b$, then when $A$ is selected, simulated and expanded, node $a$ in $C$ has its empirical average estimator updated using this newly observed data (during the Simulation of $A$) accumulated with the old data (from the Simulation of $R$). On the other hand however, for nodes $c$ and $d$, the update only involves the newly observed data in $b$ (during the Simulation of $A$).
This leads to skewed estimators, $\hat{p}\left( a\right)$ overestimates $p\left( a\right)$ while $\hat{p}\left( c\right)$ and $\hat{p}\left( d\right)$ underestimate $p\left( c\right)$ and $p\left( d\right)$, and the greater the difference in depth between $a$ and $c, d$, the greater the overestimation and underestimation effects are, we call this phenomenon \textbf{``Weights Degeneracy"}.\\
A second source of Weights Degeneracy is not copying the DT nodes, indeed $\hat{p}\left( a\right)$ is updated not only when $A$ or $C$ are simulated but also when $D$ is simulated, in the latter case, the sufficient statistics of $\hat{p}\left( c\right)$ and $\hat{p}\left( d\right)$ are not updated at all, making the overestimation/underestimation wider!

We avoid both Weights Degeneracy causes by defining a new estimator based on the chain rule:
\begin{align*}
    \prob\left[ X \in l\right] = &\prob\left[ X \in l, X \in \textrm{Parent}\left( l\right), \ldots, X \in \textrm{root}\right]\\
    =&\prob\left[ X \in l | X \in \textrm{Parent}\left( l\right)\right]\times \ldots\times\prob\left[ X \in \textrm{root}\right]
\end{align*}
We estimate each term $\prob\left[ X \in \eta | X \in \textrm{Parent}\left( \eta\right)\right]$ with $
\hat{p}\left( \eta | \textrm{Parent}\left( \eta\right)\right) = \frac{n\left( N, \eta\right)}{\sum_{\psi \in \textrm{Sib}\left( \eta\right)}n\left( N, \psi\right)}$, where $\textrm{Sib}\left( \eta\right)$ is the set of siblings of node $\eta$, including $\eta$ itself, and $n\left( N, \eta\right)$ is the number of observed samples with inputs in $\eta$. This yields the product estimator:
$$
\hat{p}\left( l\right) = \hat{p}\left( l | \textrm{Parent}\left( l\right
)\right)\times \ldots \times 1
$$
Since estimates $\hat{p}\left( \eta | \textrm{Parent}\left( \eta\right)\right)$ only involve nodes at the same depth ($\eta$ and its siblings), both sources of Weights Degeneracy are avoided, and by the Strong Law of Large numbers, we have the consistency:
$$\hat{p}\left( \eta | \textrm{Parent}\left( \eta\right)\right) \xlongrightarrow[]{\textrm{a.s}} \prob\left[ X \in \eta | X \in \textrm{Parent}\left( \eta\right)\right]$$
As more and more samples are observed in $\textrm{Parent}\left( \eta\right)$, and we deduce the consistency of our new estimator $\hat{p}\left( l\right) \xlongrightarrow[N \rightarrow \infty]{\textrm{a.s}} p\left( l\right)$.
\begin{figure}[tb]
  \centering
  \includegraphics[width=.2\textwidth]{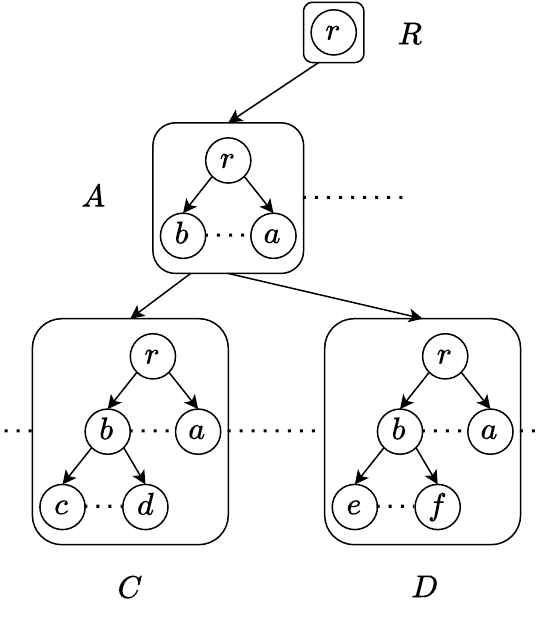}
  \caption{The Weights Degeneracy phenomenon.}
  \label{fig:degeneracy}
\end{figure}

\section{EXPERIMENTS}

\label{sec:experiments}

\begin{figure}
  \centering
  \includegraphics[width=.44\textwidth]{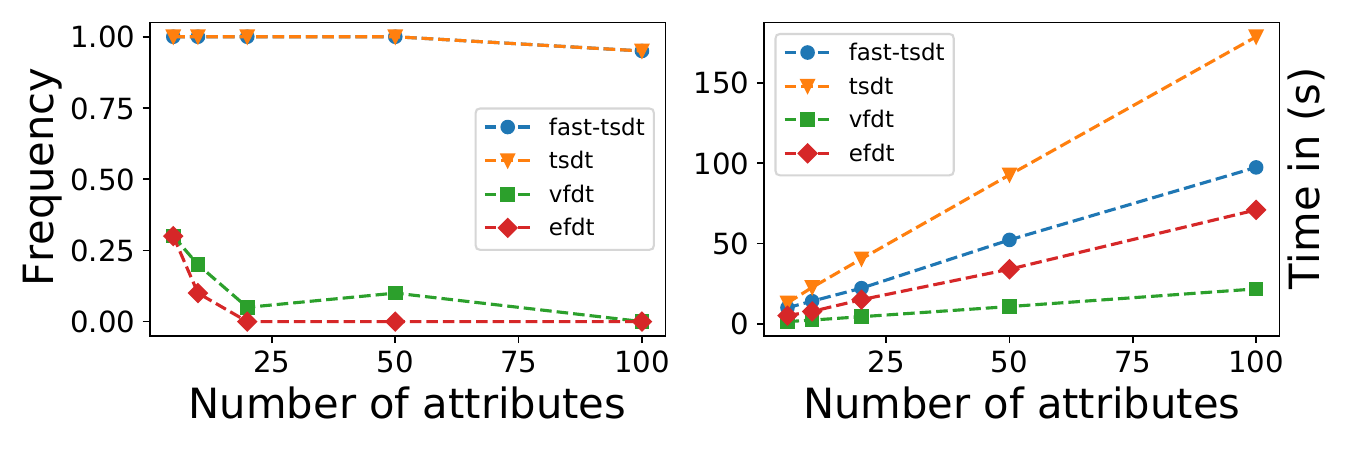}
  \caption{Comparison of VFDT, EFDT, \method~and Fast-\method. Left: Frequency of perfect convergence; Right: Average running time in seconds.}
  \label{fig:synth}
\end{figure}

\begin{figure*}[tb]
  \centering
  \includegraphics[width=0.8\textwidth]{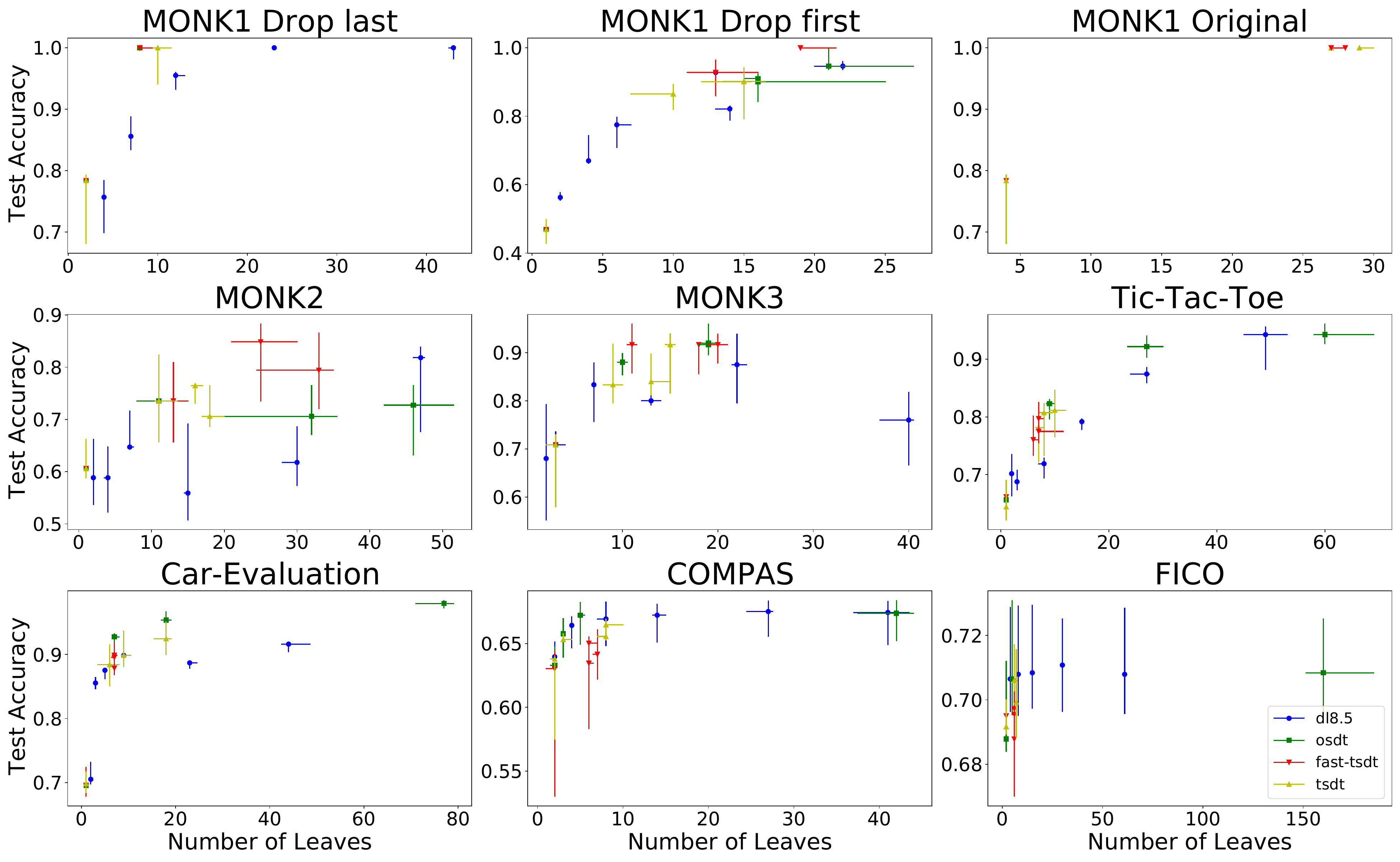}
  \caption{Cross-validation test accuracy comparison between \method, Fast-\method, \osdt~and DL8.5 as a function of the number of leaves.}
  \label{fig:test_accuracies}
\end{figure*}

Our first experiment highlights an important weakness of the classic greedy DT methods, and showcases how \method~and Fast-\method~circumvent this shortcoming. In the second experiment, we compare our methods with recent optimal batch DT algorithms DL8.5 and \osdt~on standard real world benchmarks. All the experimental details and additional results are provided in Appendix \ref{appendix:experiments}. Furthermore, our implemented code is available at \url{https://github.com/Chaoukia/Thompson-Sampling-Decision-Trees} along with the datasets we used.

We construct a challenging classification problem for greedy DT methods. In this setting, all the attributes are uninformative, in the sense that their true splitting gain metrics are all equal, making greedy methods VFDT and EFDT choose an attribute arbitrarily with tie-break. Concretely, We consider a binary classification problem with binary i.i.d. uniform attributes where $Y = 1$ if $X^{\left( 1\right)} = 0, X^{\left( 2\right)} = 0$ or $X^{\left( 1\right)} = 1, X^{\left( 2\right)} = 1$ and $Y=0$ otherwise (Figure \ref{fig:concept} provides a visualisation of the optimal DT). Attributes $X^{\left( 3\right)}, \ldots, X^{\left( q\right)}$ are irrelevant, resulting in uniform class distributions in both leaves, regardless of the attribute chosen for the root split, even when selecting one of the relevant attributes $X^{\left( 1\right)}$ or $X^{\left( 2\right)}$. Consequently, VFDT and EFDT arbitrarily choose the attribute to employ for the first split, and will continue with these arbitrary choices until attribute $X^{\left( 1\right)}$ or $X^{\left( 2\right)}$ is chosen, resulting in unnecessarily deep DTs. We compare VFDT, EFDT and \method~on settings with different numbers of attributes $q=5, 10, 20, 50, 100$. We perform 20 runs and report the frequency of perfect convergence, i.e, convergence to the optimal DT that only employs $X^{\left( 1\right)}$ and $X^{\left( 2\right)}$, and the average running time. Figure \ref{fig:synth} presents a clear trend: VFDT and EFDT rarely attain perfect convergence, especially when $q$ is large. In contrast, both \method~and Fast-\method~consistently achieve perfect convergence. Additionally, as anticipated, Fast-\method~demonstrates superior computational efficiency compared to \method.

In our second experiment, we conduct a comparison between \method, Fast-\method, \osdt, and DL8.5, even though the latter two are batch algorithms. This choice is due to the absence of prior research on optimal online DTs, to the best of our knowledge. Furthermore, the source codes of both DL8.5 and \osdt~are publicly available. Unfortunately, we could not include a comparison with the work by \cite{nunes2018monte} as the authors have not made their code publicly accessible. In this regard, we also draw attention to \citep[Table 1]{nunes2018monte}, which demonstrates how that algorithm is prohibitively slow, taking up to 105 hours in an instance. It is worth noting that our experiments exclude GOSDT because our focus is on accuracy and complexity comparison, and GOSDT is primarily an extension of \osdt~to other objective functions beyond accuracy. We follow the experimental protocol from \citep{hu2019optimal} with its datasets. We perform a 5-fold crossvalidation with different values of the hyperparameters (maximum depth for DL8.5 and $\lambda$ for \osdt~and \method), and we report in Figure \ref{fig:test_accuracies} the quartiles of the test accuracy and the number of leaves. For MONK1, \cite{hu2019optimal} utilised a One-Hot Encoding that excludes the last category of each attribute, which results in an optimal DT with 8 leaves. However, when the first category is dropped instead, the problem becomes significantly more challenging, with an optimal DT having over 18 leaves (further details are available in Appendix \ref{appendix:experiments}). On this latter problem, Figure \ref{fig:test_accuracies} clearly indicates that Fast-\method~outperforms all other methods, followed by \method. In fact, Fast-\method~is the only method that achieves $100\%$ test accuracy with 19 leaves. Moreover, unlike \osdt~and DL8.5, our methods do not necessitate binary attributes. Therefore, they can be directly applied to the original MONK1 data. In this scenario, the optimal DT representation of $Y$ with the lowest complexity consists of 27 leaves, and it is successfully retrieved by Fast-\method. \method, on the other hand, identifies a slightly more complex DT with 28 leaves. For MONK2 and MONK3, Fast-\method~demonstrates the best accuracy to number of leaves frontier. On the remaining datasets, our methods do not produce DTs with a large number of leaves, and \osdt~performs slightly better. In Appendix \ref{appendix:experiments}, we provide the execution times, DL8.5 is the fastest algorithm but also performs the least effectively, while \osdt~and \method~frequently reach their time limit of 10 minutes. On the other hand, Fast-\method~strikes a good balance between speed and performance.

\section{CONCLUSIONS, LIMITATIONS AND FUTURE WORK}

We devised \method, a new family of MCTS algorithms for constructing optimal online Decision Trees. We provided strong convergence results for our method and highlighted how it circumvents the suboptimality issue of the standard methods. Furthermore, \method~showcases similar or better accuracy-complexity trade-off compared to recent successful batch optimal DT algorithms, all while being tailored to handling data streams. For now, we are still limited to categorical attributes and our theoretical analysis only provides asymptotic convergence results. It would be desirable to derive some finite-time guarantees, in the form of PAC-bounds or rates of convergence, this is the aim of our future work. To conclude, this paper opens further possibilities for defining other MCTS algorithms with different policies, such as UCB and $\epsilon$-greedy, in the context of optimal online DTs.

\subsubsection*{Acknowledgements}

We thank Emilie Kaufmann for her insightful discussions, Otmane Sakhi for his helpful comments on the writing of the paper, and the anonymous reviewers for their valuable feedback.

\bibliography{references}

\newpage
\onecolumn
\appendix

\section{Table of Notations}

\label{appendix:notation}

\begin{table}[htbp]\caption{Table of Notation}
\begin{center}
\begin{tabular}{r c p{10cm} }
\toprule
$X$ & $=$ & $\left( X^{\left( 1\right)}, \ldots, X^{\left( q\right)}\right)$, the input.\\
$Y$ & $\in$ & $\{ 1, \ldots, K\}$, the class.\\
$T$ & $\triangleq$ & State, Decision Tree, Search Node.\\
$\overline{T}$ & $\triangleq$ & Terminal state that stems from taking the terminal action in $T$.\\
$T^{\left( l, i\right)}$ & $\triangleq$ & child Search Node of $T$ that stems from splitting leaf $l \in \leaves\left( T\right)$ with respect to attribute $X^{\left( i\right)}$.\\
$R$ & $\triangleq$ & Root Decision Tree, initial state, root of the Search Tree.\\
$\leaves\left( T\right)$ & $\triangleq$ & Set of leaves of Decision Tree $T$.\\
$X \in l$ & $\triangleq$ & Event, the subset described by $l$ contains $X$.\\
$p\left( l\right)$ & $=$ & $\prob\left[ X \in l\right]$, probability of $X \in l$.\\
$p_k\left( l\right)$ & $=$ & $\prob\left[ Y=k | X \in l\right]$, probability of $Y=k$ given $X \in l$.\\
$l\left( X\right)$ & $\triangleq$ & Leaf $l$ such that $X \in l$.\\
$\objective\left( T\right)$ & $=$ & $\prob\left[ T\left( X\right) = Y\right]$, accuracy of Decision Tree classifier $T$.\\
$T\left( X\right)$ & $\triangleq$ & $\textrm{argmax}_k\{p_k\left( l\left( X\right)\right)\}$, predicted class of $X$ according to DT $T$.\\
$\objectiver\left( T\right)$ & $\triangleq$ & Regularized Objective function of Search Node $T$.\\  
$\objectiver\left( T\right)$ & $=$ & $\prob\left[ T\left( X\right) = Y\right] - \lambda\splits\left( T\right)$.\\
$\splits\left( T\right)$ & $\triangleq$ & Number of splits in DT $T$.\\
$T \rightarrow T'$ &  $\triangleq$ & Transition from state $T$ to state $T'$.\\
$\children\left( T\right)$ & $\triangleq$ & Set of children of Search Node $T$, also set of next states from $T$.\\
$r\left( T, T'\right)$ & $\triangleq$ & Reward of transition $T \rightarrow T'$.\\
$\pi$ & $\triangleq$ & Policy, maps each state $T$ to a distribution over $\children\left( T\right)$.\\
$\pi\left( T' | T\right)$ & $\triangleq$ & Probability of transition $T \rightarrow T'$ according to policy $\pi$.\\
$\pi\left(T\right)$ & $\triangleq$ & Next state from $T$ according to policy $\pi$.\\
$T^{\left( 0\right)} \overset{\pi}{\rightarrow} \ldots, \overset{\pi}{\rightarrow} T^{\left( N\right)}$ & $\triangleq$ & Episode that stems from following policy $\pi$ starting from state $T$.\\
$\vf^\pi\left( T\right)$ & $\triangleq$ & Value of following policy $\pi$ starting from state $T$.\\
$\pi^*$ & $\in$ & $\textrm{Argmax}_\pi\vf^\pi\left( R\right)$, optimal policy at the initial state $R$.\\
$T^*$ & $\triangleq$ & Optimal Decision Tree with respect to $\objectiver$.\\
$\theta_T$ & $\triangleq$ & Posterior distribution on $\vf^{\pi^*}\left( T\right)$.\\
$\theta_{T, l}$ & $\triangleq$ & Posterior distribution on $\expec\left[ \ind\{ T\left( X\right) = Y\}|X \in l\right]$.\\
$\hat{p}\left( l\right)$ & $\triangleq$ & Estimator of $p\left( l\right)$.\\
$\hat{p}_k^{\left( i\right)}\left( l\right)$ & $=$ & $\frac{\sum_{j=1}^i\ind\{ X_j \in l\}\ind\{ Y_j = k\}}{\sum_{j=1}^i \ind\{ X_j \in l\}}$, estimator of $p_k\left( l\right)$.\\
$\hat{T}_i\left( l\right)$ & $=$ & $\textrm{Argmax}_{k}\{ \hat{p}_k^{\left( i\right)}\left( l\right)\}$, estimator of $T\left( l\right)$.\\
$\alpha_{T, l}, \beta_{T, l}$ & $=$ & Parameters of random variable $\theta_{T, l}$.\\
$\mu_{T, l}, \sigma_{T, l}$ & $=$ & Mean and standard deviation of $\theta_{T, l}$ respectively.\\
$\mu_T, \sigma_T$ & $=$ & Mean and standard deviation of $\theta_T$ respectively.\\
$n_{ ijk}\left( N, \eta\right)$ & $\triangleq$ & Given observed samples $\{ \left( X_s, Y_s\right)\}_{s=1}^N$, $n_{ ijk}\left( N, \eta\right)$ is the number samples with $X_s \in \eta, Y_s = k$ satisfying $X_s^{\left( i\right)} = j$.\\
\bottomrule
\end{tabular}
\end{center}
\label{tab:TableOfNotationForMyResearch}
\end{table}

\section{Experiments}

\label{appendix:experiments}

\begin{itemize}
    \item All of the experiments were run on a personal Machine (2,6 GHz 6-Core Intel Core i7), they are easily reproducible.
    \item The codes for DL8.5 and \osdt~ are available at \url{https://pypi.org/project/dl8.5/} and \url{https://github.com/xiyanghu/OSDT.git} respectively. We provide code for \method~and Fast-\method~at \url{https://github.com/Chaoukia/Thompson-Sampling-Decision-Trees}.
    \item In practice, the variances in Equation \eqref{eq:prior-search-leaf-2} can collapse to $0$ quickly, undermining exploration and slowing down the algorithm. We mitigate this issue by introducing an exponent $0 < \gamma < 1$ as follows:
    $$
    \left( \sigma_{\overline{T}}\right)^2 = \left( \sum_{l \in \leaves\left( T\right)} \hat{p}\left( l\right)^2 \left( \sigma_{T, l}\right)^2 \right)^\gamma
    $$
    This is also discussed in \citep[Section 3.1]{kocsis2006bandit}. The authors of UCT introduce an exponent on the bias terms $c_{t, s}$ in practice.\\
    In all our experiments, we use $\gamma=0.75$.
    \item During the Expansion step, creating all the children of the selected Search Node $T$ can be computationally expensive. Therefore, whenever $T$ is chosen, we expand it with respect to only one untreated leaf. Initially, all leaves of $T$ are marked as untreated. Upon selecting $T$, we pick the untreated leaf $l \in \leaves\left( T\right)$ with the highest Gini impurity to prioritize exploring promising parts of the Search Tree. We then generate the subset of $\children\left( T\right)$ by considering all possible split actions exclusively with respect to leaf $l$, subsequently marking $l$ as treated. We update $\pi_t$ at $T$ only when all leaves in $\leaves\left( T\right)$ have been treated. Only then, the children of $T$ in $\children\left( T\right)$ become eligible to be considered for future Selection (and the subsequent) steps.
\end{itemize}

\paragraph{Synthetic Experiment:} We use the default hyperparameters for VFDT and EFDT, for \method~and Fast-\method, we set $M=400, m=100, \lambda=0.05$.

\begin{figure}[tb]
  \centering
  \includegraphics[width=0.7\textwidth]{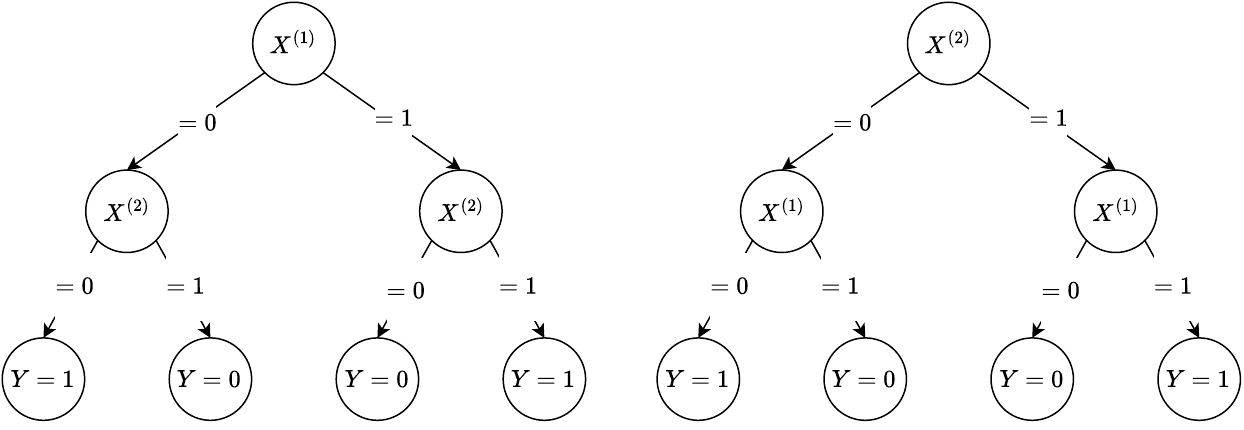}
  \caption{Equivalent representations of $Y$ as a Decision Tree for the synthetic experiment.}
  \label{fig:concept}
\end{figure}

\paragraph{Real World datasets:}
We ran both our methods with $M=1000$ iterations for MONK1 Original and MONK1 Drop Last, and with $M=10000$ iterations on the remaining datasets. All instances of \method~and Fast-\method~use $m=100$ number of samples per Simulation step. To get our accuracy-complexity frontier figures, we run the experiments with multiple values of $\lambda$ for \method, Fast-\method~and \osdt, and different values of the depth limit for DL8.5.
\begin{itemize}
    \item For DL8.5, the depth limits range from $1$ to $6$ exactly as reported by \cite{lin2020generalized}.
    \item For \method, Fast-\method~and \osdt, the set of $\lambda$ values is $0.1, 0.01, 0.0025, 0.0001$, which is a subset of the values used by \cite{lin2020generalized}. For all algorithms, we set a time limit of $10$ minutes.
\end{itemize}
In Figure \ref{fig:times}, \osdt~reaches its time limit on all the datasets except MONK1, where the last category is dropped by the Binary Encoding. \method~displays a similar behaviour but in less experiments than \osdt. On the other hand, while it is true that DL8.5 is clearly the fastest algorithm, it is also the algorithm that exhibits the least efficient accuracy-complexity frontier. In fact, when comparing the training accuracies in Figure \ref{fig:train_accuracies} and the test accuracies in Figure \ref{fig:test_accuracies}, we notice a clear overfitting trend by DL8.5 unlike the remaining algorithms. To conclude, we argue that Fast-\method~displays the best trade-off between optimality and execution time.

\begin{figure}
  \centering
  \includegraphics[width=1\textwidth]{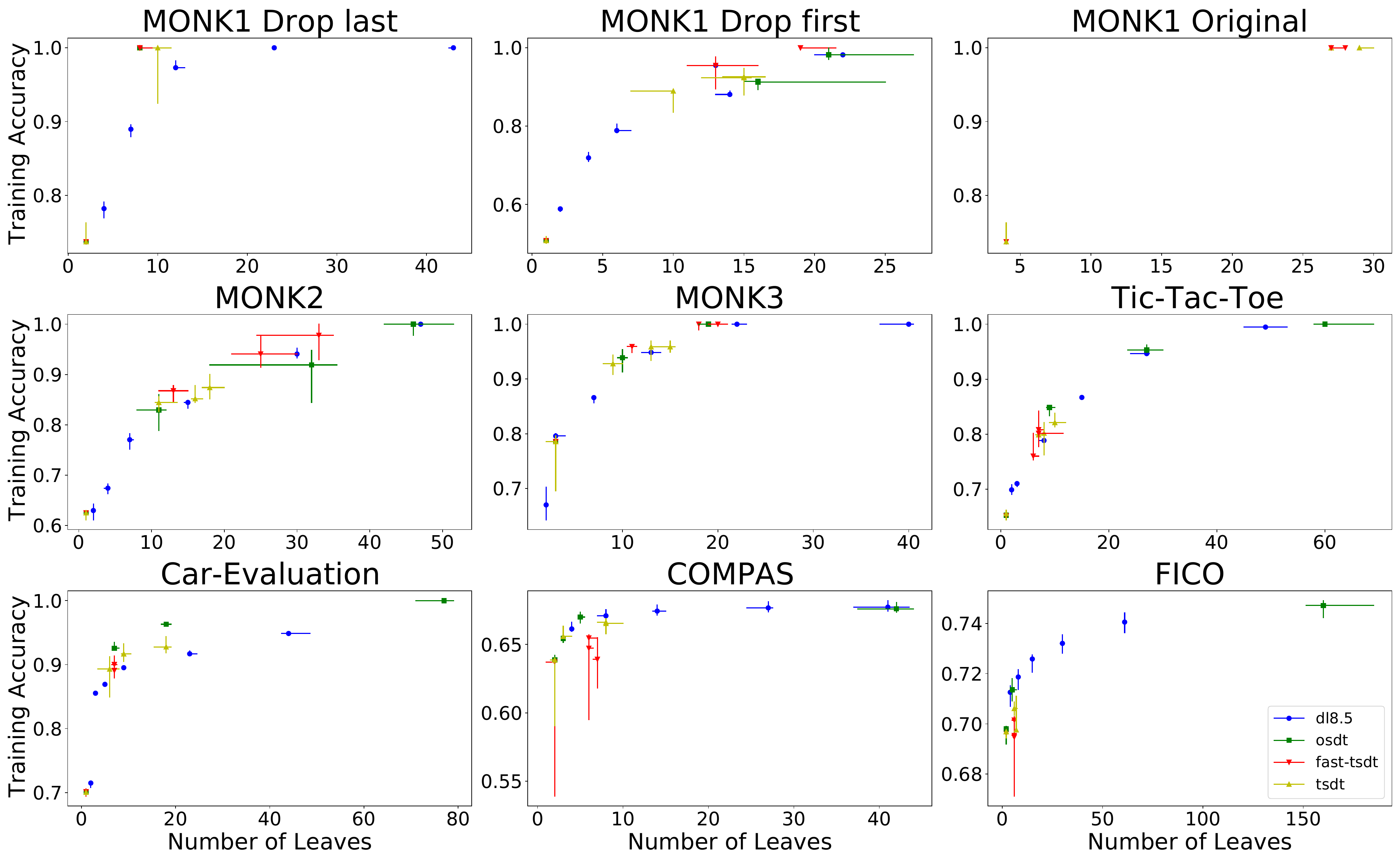}
  \caption{Cross-validation training accuracy of \method, \osdt~and DL8.5 as a function of the number of leaves.}
  \label{fig:train_accuracies}
\end{figure}

\begin{figure}
  \centering
  \includegraphics[width=1\textwidth]{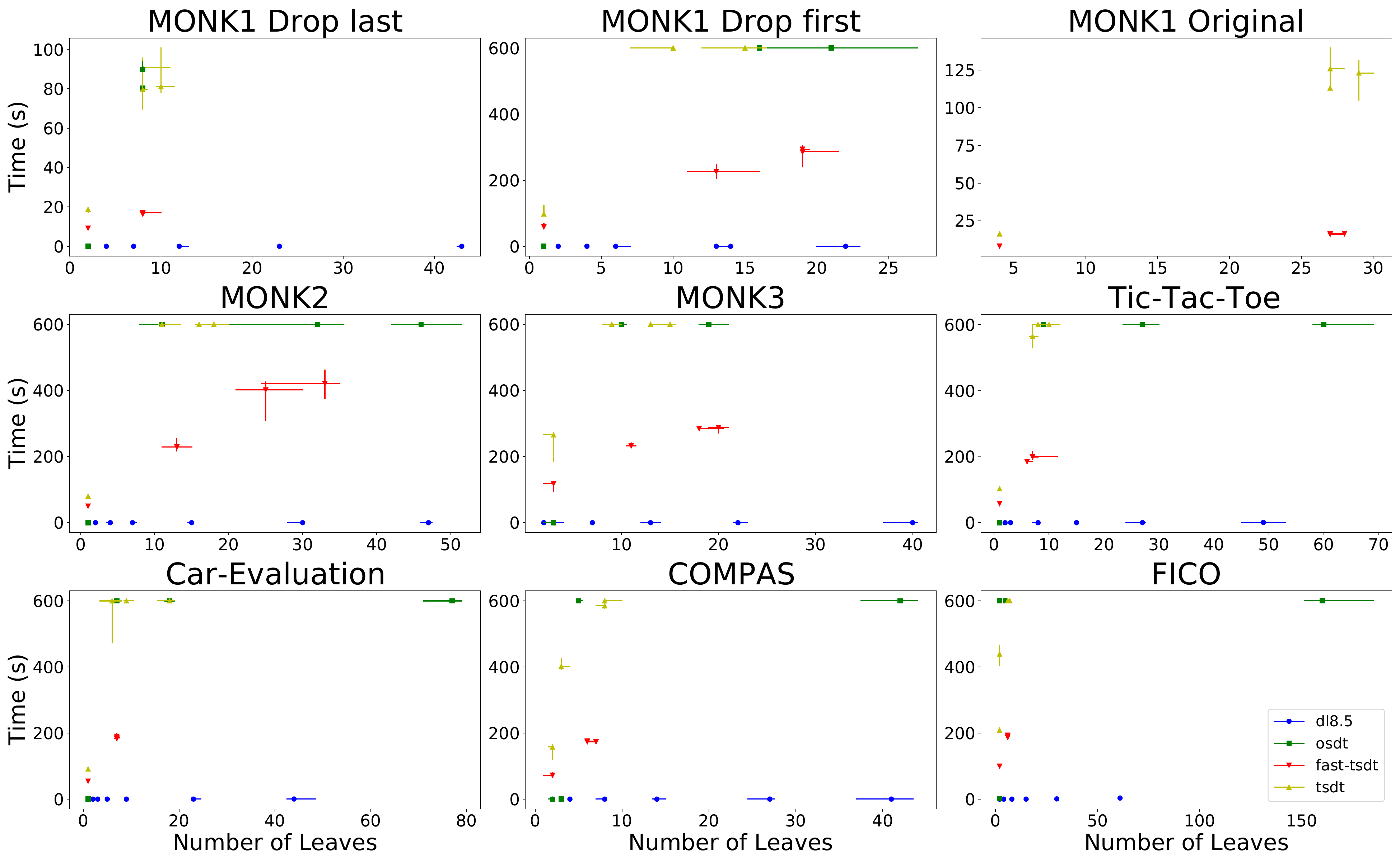}
  \caption{Cross-validation execution times of \method, \osdt~and DL8.5 as a function of the number of leaves.}
  \label{fig:times}
\end{figure}

One of the main demonstrations of \osdt was its Decision Tree solution to MONK1 Drop Last, and depicted in \citep[Figure 6]{hu2019optimal} in the comparison against BinOCT. Fast-\method~consistently achieves the same solution with 8 leaves, represented in Figure \ref{fig:tree-monk-1-last}, in much less time as documented in Figure \ref{fig:times}. Moreover, as stated in Section \ref{sec:experiments}, \cite{hu2019optimal} drop the last category of each attribute in their One-Hot Encoding of MONK1. However, without some form of prior knowledge, such choice is just arbitrary, and other options can be considered as well. The issue that arises from this is that these different options lead to solutions of different complexities as we explain in the following. In Figure \ref{fig:one-hot-encoding}, we aim at encoding some attribute variable $X$ that takes three possible values (categories). One-Hot Encoding encodes $X$ using two binary variables $X_0$ and $X_1$, if we drop the last category of $X$ then we represent $X=0$ with $X_0 = 1$, and $X=1$ with $X_1 = 1$; but to encode $X=2$ (the excluded category), we need $X_0 = 0, X_1=0$, which translates into a branch with two splits. Now, let us consider the following classification problem with $\prob\left[ Y=1 | X=2\right] = 0, \prob\left[ Y=1 | X\neq2\right] = 1$, then dropping the last category during One-Hot Encoding leads to an optimal Decision Tree with two splits, as represented in Figure \ref{fig:one-hot-encoding}, while dropping any other category instead leads to an optimal Decision Tree with only one split. In Monk 1, a similar phenomenon occurs where the choice of binary encoding has a significant impact on the complexity of the optimal DT. Indeed, dropping the last category leads to a simple solution with only 8 leaves, however, dropping the first category instead leads to a more complicated ad hard to find solution. In this case, Figures \ref{fig:train_accuracies} and \ref{fig:test_accuracies} clearly demonstrate that Fast-\method~consistently achieves a better solution than DL8.5 and \osdt. Figure \ref{fig:tree-monk-1-first} illustrates a 19-leaf solution that Fast-\method retrieves, achieving 100\% training and test accuracies, no similar solution has been found by DL8.5 and OSDT.

\begin{figure}[h]
  \centering
  \includegraphics[width=1\textwidth]{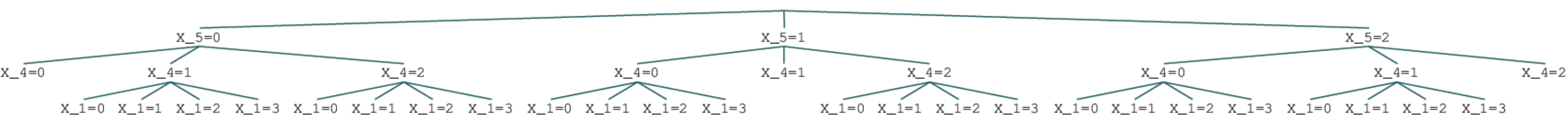}
  \caption{Fast-\method's solution to Monk 1 Original, the true optimal solution of least complexity.}
  \label{fig:tree-monk-1-original}
\end{figure}

\begin{figure}[H]
  \centering
  \includegraphics[width=0.4\textwidth]{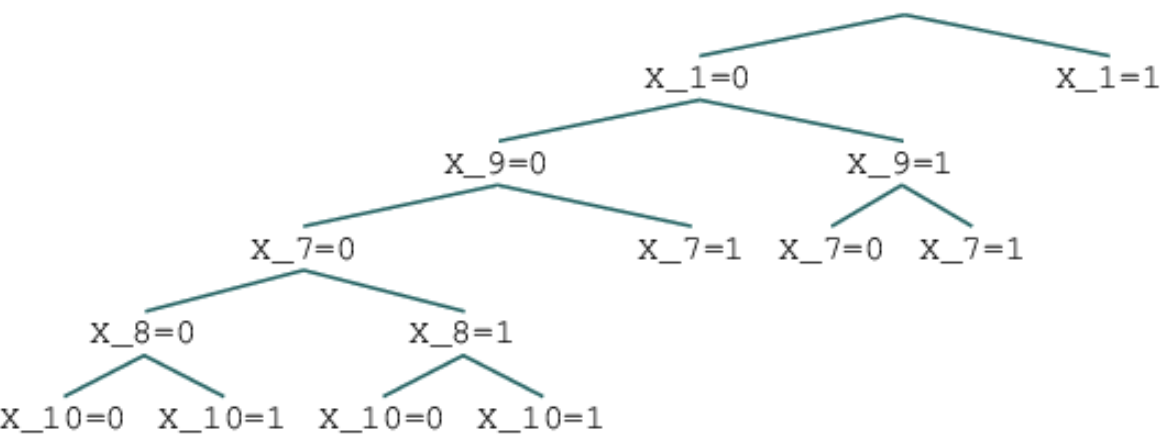}
  \caption{Fast-\method's solution to Monk 1 Drop Last, the true optimal solution of least complexity.}
  \label{fig:tree-monk-1-last}
\end{figure}

\begin{figure}[H]
  \centering
  \includegraphics[width=1\textwidth]{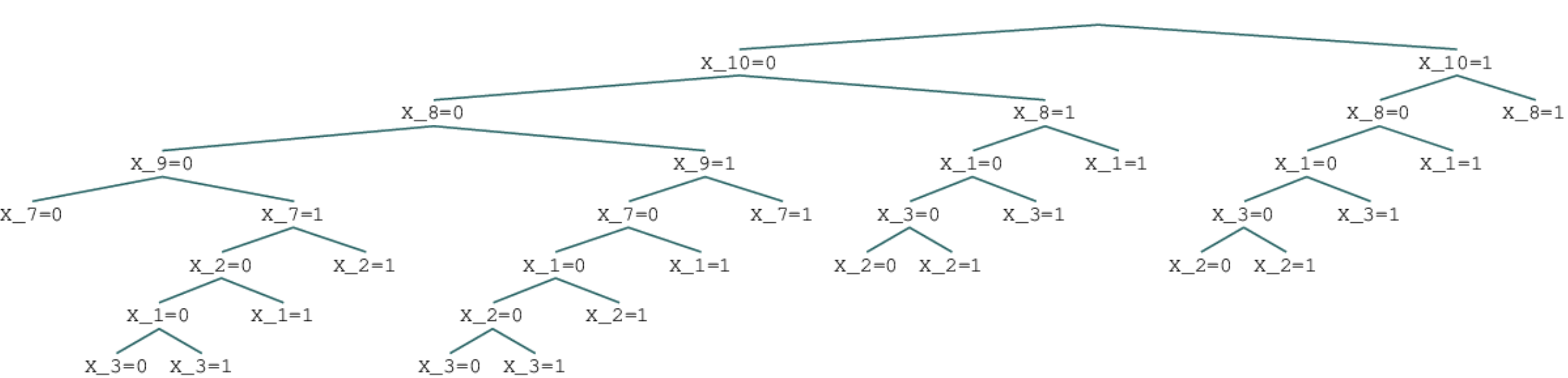}
  \caption{Fast-\method's solution to Monk 1 Drop First, this solution has $19$ leaves and achieves $100\%$ training and test accuracies; DL8.5 and \osdt~were unsuccessful in retrieving such solution.}
  \label{fig:tree-monk-1-first}
\end{figure}

\begin{figure}[H]
  \centering
  \includegraphics[width=1\textwidth]{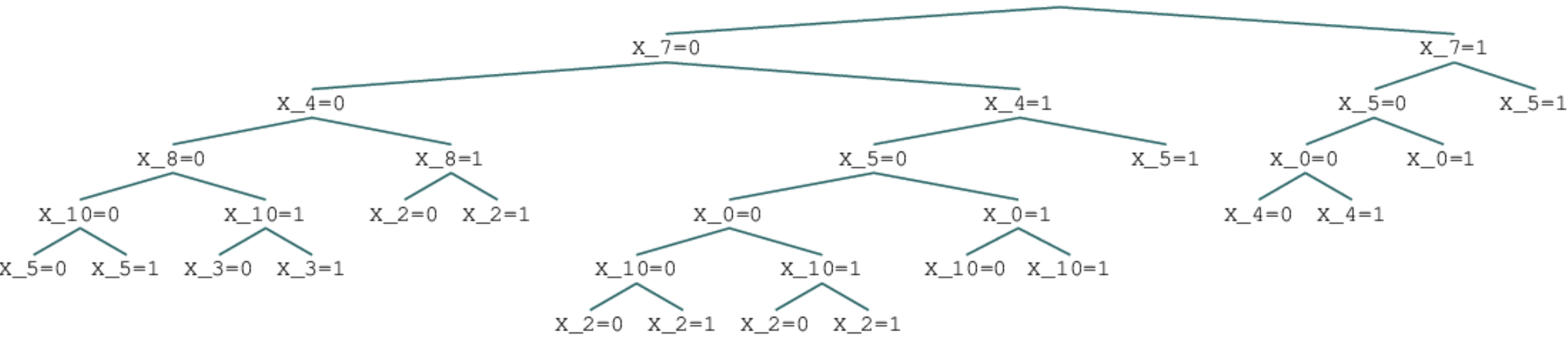}
  \caption{Fast-\method's solution to Monk 2, this solution has $17$ leaves and achieves $89.6\%$ training accuracy and $73.5\%$ test accuracies.}
  \label{fig:tree-monk-2}
\end{figure}

\begin{figure}[H]
  \centering
  \includegraphics[width=0.5\textwidth]{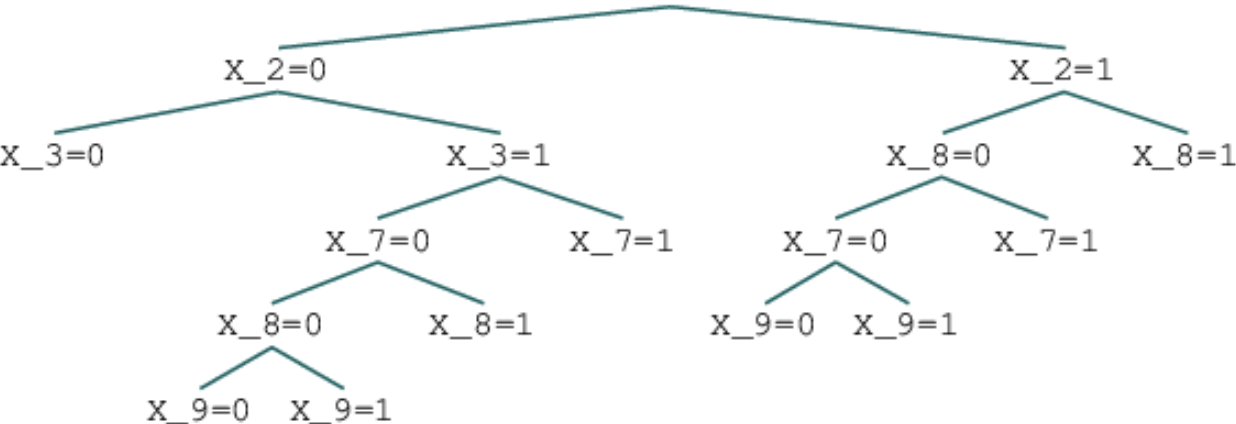}
  \caption{Fast-\method's solution to Monk 3, this solution achieves $93.8\%$ training accuracy and $92\%$ test accuracy.}
  \label{fig:tree-monk-3}
\end{figure}

\begin{figure}[H]
  \centering
  \includegraphics[width=0.5\textwidth]{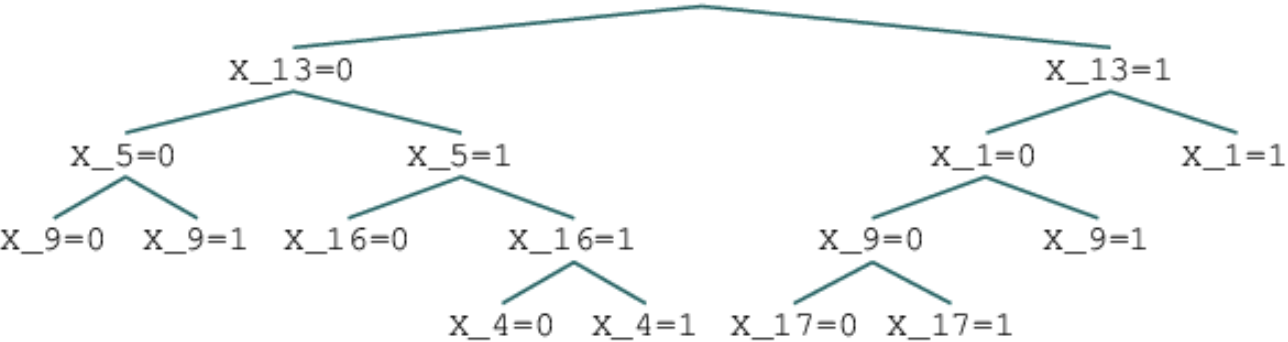}
  \caption{Fast-\method's solution to Tic-Tac-Toe, this solution achieves $79.8\%$ training accuracy and $81.8\%$ test accuracy.}
  \label{fig:tree-tic-tac-toe}
\end{figure}

\begin{figure}[H]
  \centering
  \includegraphics[width=0.5\textwidth]{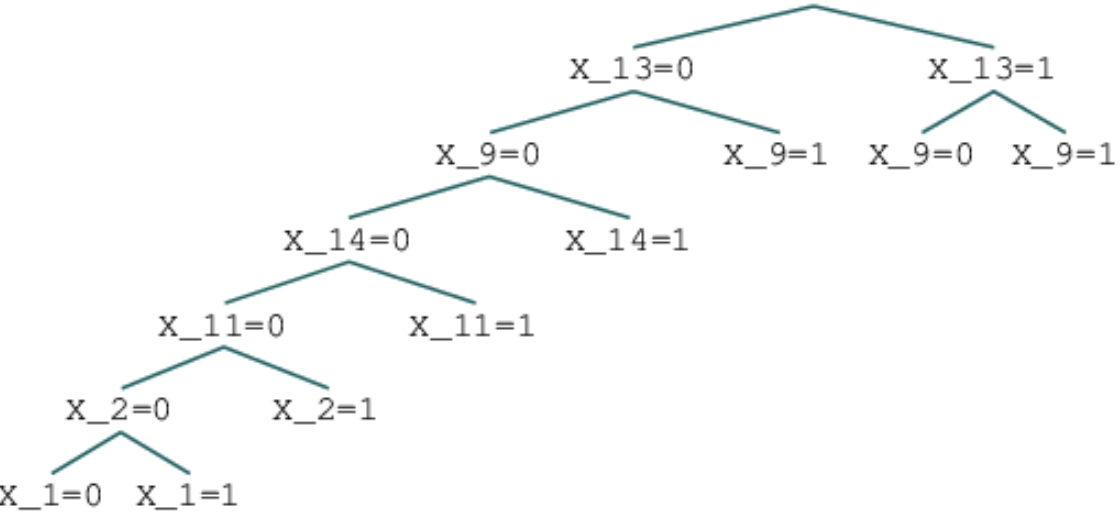}
  \caption{Fast-\method's solution to Car Evaluation, this solution achieves $89.9\%$ training accuracy and $90.8\%$ test accuracy.}
  \label{fig:tree-car-eval}
\end{figure}

\begin{figure}[H]
  \centering
  \includegraphics[width=0.4\textwidth]{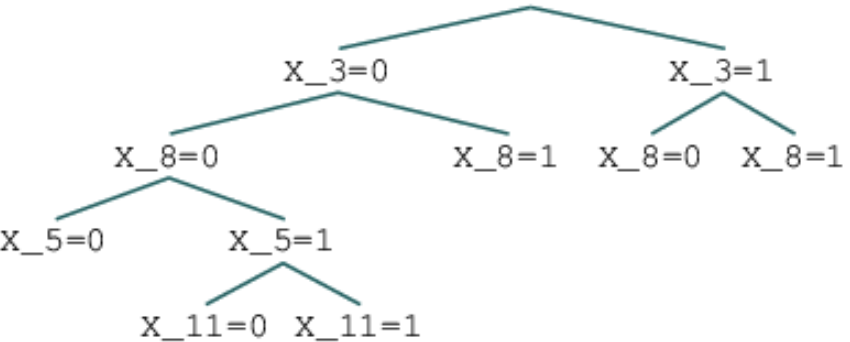}
  \caption{Fast-\method's solution to Compas, this solution achieves $66.4\%$ training accuracy and $66\%$ test accuracy.}
  \label{fig:tree-compas}
\end{figure}

\begin{figure}[H]
  \centering
  \includegraphics[width=0.4\textwidth]{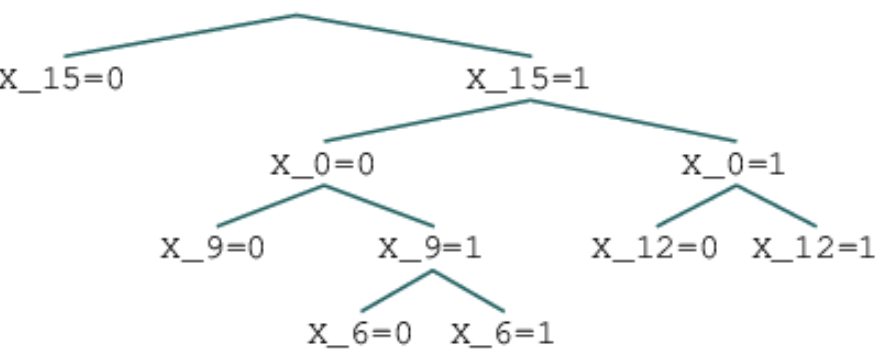}
  \caption{Fast-\method's solution to Fico, this solution achieves $69.4\%$ training accuracy and $71.8\%$ test accuracy.}
  \label{fig:tree-fico}
\end{figure}

\begin{figure}[H]
  \centering
  \includegraphics[width=0.7\textwidth]{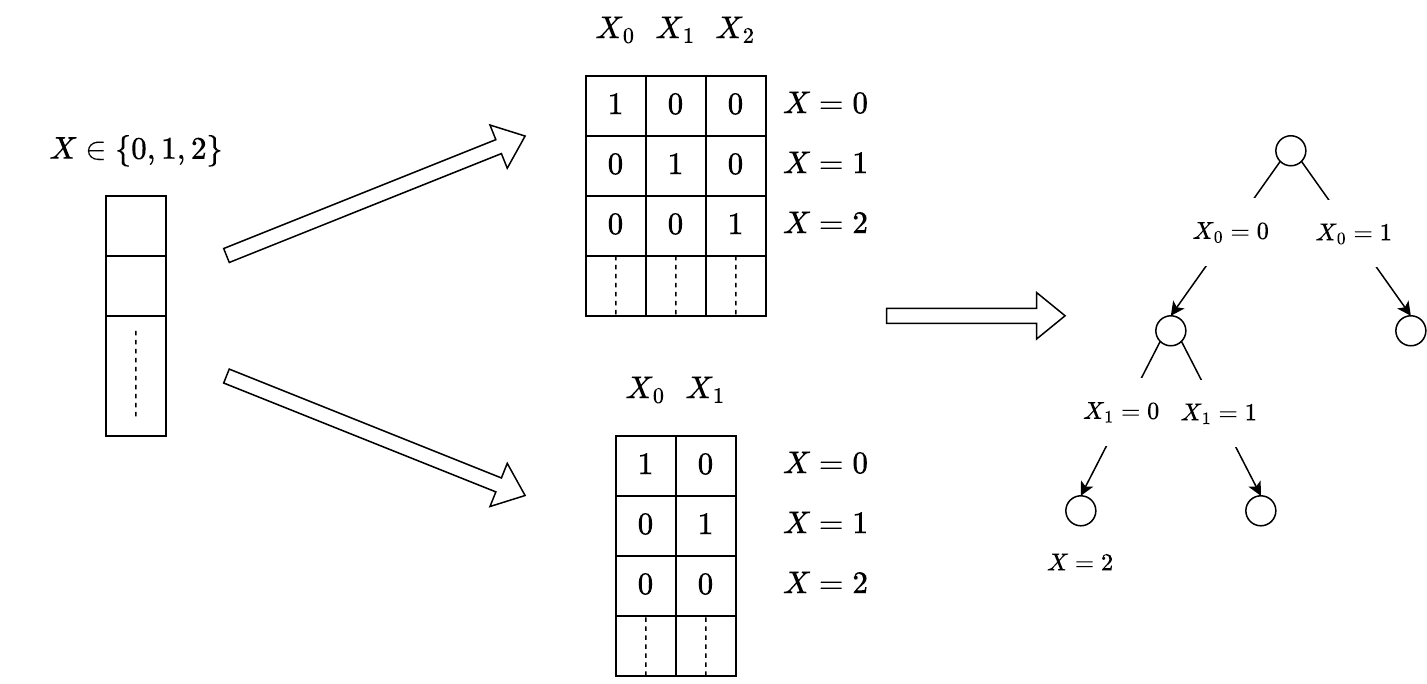}
  \caption{How the choice of which category to drop during One-Hot Encoding influences the resulting splits.}
  \label{fig:one-hot-encoding}
\end{figure}

\section{\method's Backpropagation details}

\label{appendix:backpropagation}

In this section, we present the details of the Backpropagation step performed by our first version of \method. Let $T$ be an internal Search Node, which is an internal state, and let $\children\left( T\right) = \{ T_1, \ldots, T_m\}$. We recall from Equation \eqref{eq:theta-internal} that we have:
$$
\theta_T = \max_{1 \le j \le m}\Big\{ -\lambda\ind\Big\{ T_j \neq \overline{T}\Big\} + \theta_{T_j} \Big\}
$$
Following the Inductive reasoning in Section \ref{sec:search-node}, we suppose that $\forall{1 \le j \le m}: \theta_{T_j} \sim \mathcal{N}\left( \mu_{T_j}, \left( \sigma_{T_j}\right)^2\right)$. Now, $\theta_T$ is a maximum over Normal variables that are conditionally independent given the observed data $\{ \left( X_i, Y_i\right)\}_{i=1}^N$ in $T$. We know that $\theta_T$ is not Normally distributed, but we can approximate its distribution with a Gaussian as is discussed in \citep{sinha2007advances}. Let us first consider the case with two independent Normal variables $\theta_1 \sim \mathcal{N}\left( \mu_1, \sigma_1^2\right), \theta_2 \sim \mathcal{N}\left( \mu_2, \sigma_2^2\right)$. Let $\theta = \max\{ \theta_1, \theta_2\}$, then the mean $\mu$ and variance $\sigma^2$ of $\theta$ satisfy:
\begin{align}
    \mu &= \mu_1\Phi\left( \alpha\right) + \mu_2\Phi\left( -\alpha\right) + \phi\left( \alpha\right)\sigma_m \label{eq:clark-mu} \\
    \sigma^2 &= \left( \mu_1^2 + \sigma_1^2\right)\Phi\left( \alpha\right) + \left( \mu_2^2 + \sigma_2^2\right)\Phi\left( -\alpha\right) + \left( \mu_1 + \mu_2\right)\sigma_m\phi\left( \alpha\right) - \mu^2 \label{eq:clark-sigma}
\end{align}
Where $\sigma_m = \sigma_1^2 + \sigma_2^2, \alpha = \frac{\mu_1 - \mu_2}{\sigma_m}$, and $\phi$ and $\Phi$ are respectively the probability density function and the cumulative distribution function of $\mathcal{N}\left( 0, 1\right)$ (see \citep{clark1961greatest}, \citep{10.5555/3023549.3023618}). The distribution of $\theta$ can be approximated with a Normal distribution by matching their first two moments, for short we call it Clark's approximation; \cite{sinha2007advances} provide an error analysis of this approximation. This motivates rewriting $\theta_T$ as a nested pair-wise maximum:
\begin{align*}
    \label{eq:nested}
    \theta_T &= \max\Big\{ -\lambda\ind\big\{ T_1 \neq \overline{T}\big\} + \theta_{T_1}, \max\{ -\lambda\ind\big\{ T_2 \neq \overline{T}\big\} + \theta_{T_2}, \ldots, \max\{ -\lambda\ind\big\{ T_{m-1} \neq \overline{T}\big\} + \theta_{T_{m-1}}, \\
    &-\lambda\ind\big\{ T_m \neq \overline{T}\big\} + \theta_{T_m}\}\} \ldots\Big\}
\end{align*}
$\max\{ -\lambda\ind\Big\{ T_{m-1} \neq \overline{T}\Big\} + \theta_{T_{m-1}}, -\lambda\ind\Big\{ T_m \neq \overline{T}\Big\} + \theta_{T_m}\}\}$ is a maximum of two (conditionally) independent Normal variables, thus we approximate its distribution as a Normal with Clark's approximation, and we recursively use Clark's approximation on the unfolding maximums of pairs to approximate the posterior distribution of $\theta_T$ as a Normal $\theta_T \sim \mathcal{N}\left( \mu_T, \left( \sigma_T\right)^2\right)$, $\mu_T$ and $\left( \sigma_T\right)^2$ are calculated with a recursive application of equations \eqref{eq:clark-mu} and \eqref{eq:clark-sigma}. A similar approximation was used by \cite{10.5555/3023549.3023618} in the context of MCTS with a Bayesian approach, but the policy the authors considered is not Thompson Sampling but rather a form of UCB that uses the quantiles of the posterior distributions to derive the index of the UCB policy.

\section{Proofs}

\label{appendix:proofs}

To prove Theorem \ref{thm:main}, we use and inductive reasoning that starts from Search Leaves.

\begin{lemma}
    \label{lemma:conv-search-leaves}
    Let $T$ be a Search Leaf and time $t$ the number of times $T$ is selected, then we have:
    $$
    \mu_T \xlongrightarrow[t \rightarrow \infty]{\textrm{a.s}} \vf^{\pi^*}\left( T\right), \left( \sigma_T\right)^2 \xlongrightarrow[t \rightarrow \infty]{\textrm{a.s}} 0
    $$
\end{lemma}

\begin{proof}[Proof of Lemma \ref{lemma:conv-search-leaves}]
    Let us start with the variance as it is a simpler result to prove. From the definition of $\alpha_T\left( t\right)$ and $\beta_T\left( t\right)$, we have $\alpha_T\left( t\right) + \beta_T\left( t\right) = tm + 1$, where we recall that $m$ is the number of observed samples from the stream in a single Simulation step. Therefore we have:
    \begin{align*}
        \left( \sigma_T\right)^2 &=  \frac{\alpha_T\left( t\right)\beta_T\left( t\right)}{\left( \alpha_T\left( t\right) + \beta_T\left( t\right)\right)^2\left( \alpha_T\left( t\right) + \beta_T\left( t\right) + 1\right)}\\
        &= \frac{\alpha_T\left( t\right)\beta_T\left( t\right)}{\left( tm + 1\right)^2\left( tm + 2\right)}\\
        &\le \frac{\left( tm + 1\right)^2}{\left( tm + 1\right)^2\left( tm + 2\right)}\\
        &\le \frac{1}{tm + 2}
    \end{align*}
    Therefore $\left( \sigma_T\right)^2$ converges, not only almost surely, but even surely to $0$ as $t \rightarrow \infty$.

    Now, let us show the result for the mean, we have:
    $$\mu_T = \frac{\alpha_T\left( t\right)}{tm + 1} = \frac{1}{1 + tm} + \frac{1}{1 + tm}\sum_{i=2}^{tm}\ind\Big\{ \hat{T}_{i-1}\left( X_i\right) = Y_i\Big\}$$
    We know that, for any $l \in \leaves\left( T\right)$ such that $p\left( l\right) > 0$, which are the only leaves that can contain inputs $X_i$, the number of observed samples grows to infinity almost surely as $t \rightarrow \infty$, therefore, by the Strong Law of Large numbers we have:
    $$
    \forall{k}\in \{ 1 \ldots, K\}: \hat{p}_k^{\left( i\right)}\left( l\right) \xlongrightarrow[i \rightarrow \infty]{\textrm{a.s}} p_k\left( l\right)
    $$
    Which means that:
    \begin{equation}
        \prob\left[ \forall{\epsilon}>0, \exists I > 0, \forall{i} \ge I, \forall{k} \in \{ 1, \ldots, K\}: \Big| \hat{p}_k^{\left( i\right)}\left( l\right) - p_k\left( l\right)\Big| \le \epsilon \right] = 1 \label{eq:p_k-as}
    \end{equation}
    Take $\epsilon = \frac{1}{2}\min_{k \neq k'}\Big\{ \Big| p_k\left( l\right) - p_{k'}\left( l\right)\Big|\Big\}$, in light of Equation \eqref{eq:p_k-as} we have:
    \begin{align}
        &\prob\left[ \exists \tau > 0, \forall{i} \ge \tau, \forall{k} \in \{ 1, \ldots, K\}: \Big| \hat{p}_k^{\left( i\right)}\left( l\right) - p_k\left( l\right)\Big| \le \epsilon \right] = 1 \nonumber\\
        \implies & \prob\left[ \exists \tau > 0, \forall{i} \ge \tau: \textrm{Argmax}_k \{ p_k\left( l\right)\} = \textrm{Argmax}_k \{ \hat{p}_k^{\left( i\right)}\left( l\right)\} \right] = 1 \nonumber\\
        \implies & \prob\left[ \exists \tau > 0, \forall{i} \ge \tau: \hat{T}_i\left( l\right) = T\left( l\right)\right] = 1 \label{eq:T_as}
    \end{align}
    Let us define $\tau > 0$ the random time such that:
    $$
    \forall{i} \ge \tau: \hat{T}_{i-1}\left( X_i\right) = T\left( X_i\right)
    $$
    Since Equation \eqref{eq:T_as} is satisfied for all leaves $l \in \leaves\left( T\right)$, we have $\prob\left[ \tau < \infty\right] = 1$.\\
    Let $0 < \epsilon' < \epsilon$, we want to show that:
    $$
    \prob\left[ \exists \tau' > 0, \forall{t} > \tau': \Big| \mu_T - \vf^{\pi^*}\left( T\right)\Big| \le \epsilon'\right] = 1
    $$
    By marginalising over $\tau > 0$, we have:
    \begin{align}
        &\prob\left[ \exists \tau' > 0, \forall{t} > \tau':  \Big| \mu_T - \vf^{\pi^*}\left( T\right)\Big| \le \epsilon'\right]\\
        & = \sum_{t' > 0}\prob\left[ \exists \tau' > 0, \forall{t} > \tau':  \Big| \mu_T - \vf^{\pi^*}\left( T\right)\Big| \le \epsilon' \Big| \tau = t'\right]\prob\left[ \tau = t'\right] \label{eq:margin}
    \end{align}
    Note that in this marginalisation, the term $\prob\left[ \tau = \infty\right]$ is absent because $\prob\left[ \tau = \infty\right] = 0$.\\
    From the definition of $\tau$, we can write $\mu_T$ as follows:
    $$
    \forall{t} > \tau: \mu_T = \frac{1}{1 + tm} + \frac{1}{1 + tm}\sum_{i=2}^{\tau m}\ind\Big\{ \hat{T}_{i-1}\left( X_i\right) = Y_i\Big\} + \frac{1}{1 + tm}\sum_{i=\tau m + 1}^{tm}\ind\Big\{ T\left( X_i\right) = Y_i\Big\}
    $$
    Now, given $\tau = t'$, then for all $t > t'$ we have the following:
    $$
    \Big| \mu_T - \vf^{\pi^*}\left( T\right)\Big| \le \Big| \frac{1}{1 + tm} + \frac{1}{1 + tm}\sum_{i=2}^{t' m}\ind\Big\{ \hat{T}_{i-1}\left( X_i\right) = Y_i\Big\}\Big| + \Big| \frac{1}{1 + tm}\sum_{i=t' m + 1}^{tm}\ind\Big\{ T\left( X_i\right) = Y_i\Big\} - \vf^{\pi^*}\left( T\right)\Big|
    $$
    For the first term of the RHS, we have:
    $$
    \Big| \frac{1}{1 + tm} + \frac{1}{1 + tm}\sum_{i=2}^{t' m}\ind\Big\{ \hat{T}_{i-1}\left( X_i\right) = Y_i\Big\}\Big| \le \Big| \frac{1}{1 + tm} + \frac{t'm - 1}{1 + tm}\Big| \underset{t \rightarrow \infty}{\longrightarrow} 0
    $$
    Hence, there exists $\tau_1 > 0$ such that:
    $$
    \forall{t} > \tau_1: \Big| \frac{1}{1 + tm} + \frac{1}{1 + tm}\sum_{i=2}^{t' m}\ind\Big\{ \hat{T}_{i-1}\left( X_i\right) = Y_i\Big\}\Big| \le \frac{\epsilon'}{3}
    $$
    For the second term of the RHS, we bound it first as follows:
    \begin{align*}
        \Big| \frac{1}{1 + tm}\sum_{i=t' m + 1}^{tm}\ind\Big\{ T\left( X_i\right) = Y_i\Big\} - \vf^{\pi^*}\left( T\right)\Big| &\le \Big| \frac{1}{tm - t'm}\sum_{i=t' m + 1}^{tm}\ind\Big\{ T\left( X_i\right) = Y_i\Big\} - \vf^{\pi^*}\left( T\right)\Big|\\
        &+ \Big| \left( \frac{1}{1 + tm} - \frac{1}{tm - t'm}\right)\sum_{i=t'm + 1}^{tm}\ind\Big\{ T\left( X_i\right) = Y_i\Big\}\Big|
    \end{align*}
    By the Strong Law of Large numbers, with probability $1$, there exists $\tau_2 > 0$ such that:
    $$
    \forall{t} > \tau_2: \Big| \frac{1}{tm - t'm}\sum_{i=t' m + 1}^{tm}\ind\Big\{ T\left( X_i\right) = Y_i\Big\} - \vf^{\pi^*}\left( T\right)\Big| \le \frac{\epsilon'}{3}
    $$
    On the other hand:
    $$
    \Big| \left( \frac{1}{1 + tm} - \frac{1}{tm - t'm}\right)\sum_{i=t'm + 1}^{tm}\ind\Big\{ T\left( X_i\right) = Y_i\Big\}\Big| \le \frac{t'm+1}{1 + tm} \underset{t \rightarrow \infty}{\longrightarrow} 0
    $$
    Therefore, there exists $\tau_3 > 0$ such that:
    $$
    \forall{t > \tau_3}: \Big| \left( \frac{1}{1 + tm} - \frac{1}{tm - t'm}\right)\sum_{i=t'm + 1}^{tm}\ind\Big\{ T\left( X_i\right) = Y_i\Big\}\Big| \le \frac{\epsilon'}{3}
    $$
    Now take $\tau' = \max\{\tau_1, \tau_2, \tau_3\}$, then we have:
    $$
    \Big| \mu_T - \vf^{\pi^*}\left( T\right)\Big| \le \frac{\epsilon'}{3} + \frac{\epsilon'}{3} + \frac{\epsilon'}{3} = \epsilon'
    $$
    Therefore, conditionally on $\tau = t'$, with probability $1$, there exists $\tau' > 0$ such that $\forall{t} > \tau': \Big| \mu_T - \vf^{\pi^*}\left( T\right)\Big| \le \epsilon'$, i.e:
    $$
    \prob\left[ \exists \tau' > 0, \forall{t} > \tau': \Big| \mu_T - \vf^{\pi^*}\left( T\right)\Big| \le \epsilon' \Big| \tau = t'\right] = 1
    $$
    From Equation \eqref{eq:margin}, we deduce that:
    \begin{align*}
        &\forall{0 < \epsilon' < \epsilon}:
        \prob\left[ \exists \tau' > 0, \forall{t} > \tau': \Big| \mu_T - \vf^{\pi^*}\left( T\right)\Big| \le \epsilon'\right] = \sum_{t' > 0}\prob\left[ \tau = t'\right] = 1\\
        \implies & \prob\left[ \forall{0 < \epsilon' < \epsilon}, \exists \tau' > 0, \forall{t} > \tau': \Big| \mu_T - \vf^{\pi^*}\left( T\right)\Big| \le \epsilon'\right] = 1\\
        \implies & \mu_T \xlongrightarrow[t \rightarrow \infty]{\textrm{a.s}} \vf^{\pi^*}\left( T\right)
    \end{align*}
\end{proof}

\begin{lemma}
    \label{lemma:visits-conv}
    For any Search Node $T$ with $t$ the number of visits of its parent and $N_T\left( t\right)$ the number of times $T$ has been visited up to time $t$, we have:
    $$
    N_T\left( t\right) \xlongrightarrow[t \rightarrow \infty]{\textrm{a.s}} \infty
    $$
\end{lemma}

\begin{proof}[Proof of Lemma \ref{lemma:visits-conv}]
    Let $P\left( T\right)$ denote the parent of $T$, for convenience, we denote $\children\left( P\left( T\right)\right) = \{ T_1, \ldots, T_n\}$ and $T = T_j$, now we want to show that $N_{T_j}\left( t\right) \xlongrightarrow[t \rightarrow \infty]{\textrm{a.s}} \infty$. To do so, we want to prove the following result: $\prob\left[ \theta_{T_j} \le \max_{i \neq j}\{ \theta_{T_i}\}\right] < 1$ at all times $t$, i.e, the probability of choosing $T_j$ is always non-zero. We consider the auxiliary problem with $\children'\left( P\left( T\right)\right) = \{ T_j, T'\}$ where 
    \begin{equation*}
        \theta_{T'} \sim \mathcal{N}\left( \mu_{T'}, \left( \sigma_{T'}\right)^2\right), \mu_{T'} = \max_{i \neq j}\{ \mu_{T_i}\} + f_n\left( t\right), \sigma_{T'} = \max_{i \neq j}\{ \sigma_{T_i}\}
    \end{equation*}
    Where we will define $f_n\left( t\right)$ such that $\prob\left[ \theta_{T'} \ge \max_{i \neq j}\{ \theta_{T_i}\}\right] \ge \frac{1}{2}$ at all times $t$. In this case, we have the following bound: $
    \prob\left[ \theta_{T_j} \le \max_{i \neq j}\{ \theta_{T_i}\}\right] \le \prob\left[ \theta_{T_j} \le \theta_{T'}\right]$. We use the union bound:
    \begin{align*}
    \prob\left[ \theta_{T'} \ge \max_{i \neq j}\{ \theta_{T_i}\}\right] \ge 1 - \sum_{i \neq j} \prob\left[ \theta_{T'} < \theta_{T_i}\right]
    \end{align*}
    Since $\forall{i}\neq j: \theta_{T'} - \theta_{T_i} \sim \mathcal{N}\left( \mu_{T'} - \mu_{T_i}, \left( \sigma_{T'}\right)^2 + \left( \sigma_{T_i}\right)^2\right)$, we have:
    \begin{align*}
        \prob\left[ \theta_{T'} < \theta_{T_i}\right] &= \frac{1}{2}\textrm{erfc}\left( \frac{\max_{k \neq j}\{ \mu_{T_k}\} - \mu_{T_i} + f_n\left( t\right)}{\sqrt{2\left[ \left( \sigma_{T'}\right)^2 + \left( \sigma_{T_i}\right)^2\right]}}\right)\\
        &\le \frac{1}{2}\textrm{erfc}\left( \frac{f_n\left( t\right)}{2\sigma_{T'}}\right)
    \end{align*}
    Hence:
    \begin{equation*}
        \prob\left[ \theta_{T'} \ge \max_{i \neq j}\{ \theta_{T_i}\}\right] \ge 1 - \frac{n-1}{2}\textrm{erfc}\left( \frac{f_n\left( t\right)}{2\sigma_{T'}}\right)
    \end{equation*}
    Thus, we want $f_n\left( t\right)$ satisfying $\textrm{erfc}\left( \frac{f_n\left( t\right)}{2\sigma_{T'}}\right) \le \frac{1}{n-1}$.\\
    Take $f_n\left( t\right) = g_n\left( t\right)\sigma_{T'}$, hence it suffices to take $g_n\left( t\right) = 2\textrm{erfc}^{-1}\left( \frac{1}{n-1}\right)$ and thus $f_n\left( t\right) = 2\sigma_{T'}\textrm{erfc}^{-1}\left( \frac{1}{n-1}\right)$.

    On the other hand, we have the following:
    $$
    \prob\left[ \theta_{T_j} > \theta_{T'}\right] = \frac{1}{\sqrt{\pi}}\bigintsss_{\frac{\mu_{T'} - \mu_{T_j}}{\sqrt{2\left[ \left( \sigma_{T'}\right)^2 + \left( \sigma_{T_j}\right)^2\right]}}}^{\infty} e^{-u^2}du > 0
    $$
    Therefore we deduce that:
    $$
    \forall{t} > 0: \prob\left[ \theta_{T_j} \le \max_{i \neq j}\{ \theta_{T_i}\}\right] \le \prob\left[ \theta_{T_j} \le \theta_{T'}\right] < 1
    $$
    To show that $N_{T_j}\left( t\right) \xlongrightarrow[t \rightarrow \infty]{\textrm{a.s}} \infty$, we will equivalently prove:
    $$
    \prob\left[ \exists{M} > 0, \forall{t} > 0: N_{T_j}\left( t\right) < M\right] = 0
    $$
    The event $\{ \exists{M} > 0, \forall{t} > 0: N_{T_j}\left( t\right) < M\}$ can be rewritten as the event $\{ \exists{\tau > 0}, \forall{t \ge \tau}: \theta_{T_j} \le \max_{i \neq j}\{ \theta_i\}\}$, which means that there exists a time $\tau > 0$ such that, from then on, $T_j$ will never be chosen again. Therefore:
    \begin{align*}
        \prob\left[ \exists{M} > 0, \forall{t} > 0: N_{T_j}\left( t\right) < M\right] &= \prob\left[ \exists{\tau > 0}, \forall{t \ge \tau}: \theta_{T_j} \le \max_{i \neq j}\{ \theta_i\}\right]\\
        &\le \sum_{\tau > 0}\prob\left[ \forall{t \ge \tau}: \theta_{T_j} \le \max_{i \neq j}\{ \theta_i\}\right]\\
        &\le \sum_{\tau > 0}\prod_{t \ge \tau}\prob\left[ \theta_{T_j} \le \max_{i \neq j}\{ \theta_i\}\right]\\
        &\le \sum_{\tau > 0}\prod_{t \ge \tau}\prob\left[ \theta_{T_j} \le \theta_{T'}\right]
    \end{align*}
    The third line comes from the fact that $\{\theta_{T_i}\}$ are independent. We recall that for all $t \ge \tau$:
    \begin{align*}
        \prob\left[ \theta_{T_j} \le \theta_{T'}\right] &= \frac{1}{\sqrt{\pi}}\bigintsss_{-\infty}^{\frac{\mu_{T'} - \mu_{T_j}}{\sqrt{2\left[ \left( \sigma_{T'}\right)^2 + \left( \sigma_{T_j}\right)^2\right]}}} e^{-u^2}du \le \frac{1}{\sqrt{\pi}}\bigintsss_{-\infty}^{\frac{1 + f_n\left( \tau\right) - \mu_{T_j}}{\sqrt{2}\sigma_{T_j}}} e^{-u^2}du < 1
    \end{align*}
    This comes from the fact that $\forall{1 \le i \le n}: 0 \le \mu_i \le 1$ and $f_n\left( t\right)$ is a decreasing function of $t$ because $\sigma_{T'}$ decreases with $t$. Since $\forall{t \ge \tau}: T_j$ is not chosen, then $\mu_{T_j}$ and $\sigma_{T_j}$ remain constant, and therefore:
    $$
    \prod_{t \ge \tau}\frac{1}{\sqrt{\pi}}\bigintsss_{-\infty}^{\frac{1 + f_n\left( \tau\right) - \mu_{T_j}}{\sqrt{2}\sigma_{T_j}}} e^{-u^2}du = 0
    $$
    Thus we deduce that:
    $$
    \prob\left[ \exists{M} > 0, \forall{t} > 0: N_{T_j}\left( t\right) < M\right] = 0
    $$
    Which concludes our proof.
\end{proof}

\begin{corollary}
    \label{cor:visits-conv}
    Let $t$ be the number of iterations of \method~, then the number of visits of any Search Node diverges almost surely to $\infty$ as $t \rightarrow \infty$.
\end{corollary}

\begin{proof}[Proof of Corollary \ref{cor:visits-conv}]

Corollary \ref{cor:visits-conv} is straightforward to prove by Induction using Lemma \ref{lemma:visits-conv} and the fact that $t$ is the number of visits of the Root Search Node.
    
\end{proof}

In what follows, let $T$ denote an internal Search Node with $\children\left( T\right) = \{ T_1, \ldots, T_n\}$ and $T_1$ the optimal child, i.e $\vf^{\pi^*}\left( T_1\right) = \max_{1 \le j \le n}\{ \vf^{\pi^*}\left( T_j\right)\}$.

\begin{lemma}
    \label{lemma:proba-opt}
    Let time $t$ denote the number of times that $T$ has been visited. Suppose that $ \forall{j} \in \{ 1, \ldots, n\}: \vf^{\pi^*}\left( T_j\right) \xlongrightarrow[N_{T_j}\left( t\right) \rightarrow \infty]{\textrm{a.s}} \mu_{T_j}, \sigma_{T_j} \xlongrightarrow[N_{T_j}\left( t\right) \rightarrow \infty]{\textrm{a.s}} 0$ , then we have:
    $$
    \lim_{t \rightarrow\infty}\pi_t\left( T_1 | T\right) = 1
    $$
    \emph{Note that we abuse the notation a little bit here. Indeed, in the main paper $\pi_t$ is the policy after $t$ iterations of \method, but here $\pi_t$ denotes the policy after choosing $T$ for $t$ times.}
\end{lemma}

\begin{proof}[Proof of Lemma \ref{lemma:proba-opt}]
    We define the following events at time $t$:
    \begin{align*}
        \mathcal{M}\left( t, \epsilon\right) &= \bigcap_{j=1}^n\Bigg\{ \Bigg|\mu_{T_j} - \vf^{\pi^*}\left( T_j\right)\Bigg| \le \epsilon \Bigg\}\\
        \mathcal{V}\left( t, \epsilon\right) &= \bigcap_{j=1}^n\Bigg\{ \left( \sigma_{T_j}\right)^2 < \frac{\epsilon}{2} \Bigg\}
    \end{align*}
    and $t\left( \epsilon\right) > 0$ the random time such that $\forall{t}> t\left( \epsilon\right):$ $\mathcal{M}\left( t, \epsilon\right)$ and $\mathcal{V}\left( t, \epsilon\right)$ happen. By Lemma \ref{lemma:visits-conv}, we have $\forall{j} \in \{ 1, \ldots, n\}: N_{T_j}\left( t\right)\xlongrightarrow[t \rightarrow \infty]{\textrm{a.s}} \infty$, therefore $\prob\left[ t\left( \epsilon\right) < \infty\right] = 1$. Let $i\left( t\right)$ denote the chosen child at time $t$, we have $\prob\left[ i\left( t\right) = T_1\right] = \pi_t\left( T_1 | T\right)$. The introduction of $i\left( t\right)$ is purely for convenience purposes. We write:
    \begin{align*}
        \prob\left[ i\left( t\right) = T_1\right] &= \sum_{\tau > 0}\prob\left[ i\left( t\right) = T_1 \Big| t\left( \epsilon\right) = \tau\right]\prob\left[ t\left( \epsilon\right) = \tau\right]\\
        &\ge \sum_{1 \le \tau \le t}\prob\left[ i\left( t\right) = T_1 \Big| t\left( \epsilon\right) = \tau\right]\prob\left[ t\left( \epsilon\right) = \tau\right]
    \end{align*}
    Conditionally on $t\left( \epsilon\right) = \tau$, for $t \ge \tau$ we have the following:
    \begin{align*}
        \prob\left[ i\left( t\right) = T_1 \Big| t\left( \epsilon\right)\right] &= \prob\left[ \forall{j} \neq 1: \theta_{T_1} > \theta_{T_j} \Big| \mathcal{M}\left( t, \epsilon\right), \mathcal{V}\left( t, \epsilon\right)\right]\\
        &\ge 1 - \sum_{j \neq 1}\prob\left[ \theta_{T_1} \le \theta_{T_j} \Big| \mathcal{M}\left( t, \epsilon\right), \mathcal{V}\left( t, \epsilon\right)\right]
    \end{align*}
    Before we continue, we introduce the following notation $\forall{j}\neq 1: \Delta_j = \vf^{\pi^*}\left( T_1\right) - \vf^{\pi^*}\left( T_j\right) > 0, \Delta = \min_{j \neq 1} \Delta_j$. Let $C>0$ and define $\epsilon = \frac{\Delta}{4C}$, then for all $j \neq 1$:
    \begin{align*}
        &\prob\left[ \theta_{T_1} \le \theta_{T_j} \Big| \mathcal{M}\left( t, \epsilon\right), \mathcal{V}\left( t, \epsilon\right)\right]\\
        &\le \prob\left[ \theta_{T_1} \le \theta_{T_j} \Bigg| \Big| \mu_{T_1} - \vf^{\pi^*}\left( T_1\right) \Big| < \frac{\Delta_j}{4C}, \Big| \mu_{T_j} - \vf^{\pi^*}\left( T_j\right) \Big| < \frac{\Delta_j}{4C}, \left( \sigma_{T_1}\right)^2 < \frac{\Delta_j}{8C}, \left( \sigma_{T_j}\right)^2 < \frac{\Delta_j}{8C} \right]\\
        &\le \frac{1}{\sqrt{\pi}}\bigintsss_{4C}^{\infty} e^{-u^2}du = \frac{1}{2}\textrm{erfc}\left( 4C\right)
    \end{align*}
    Hence, we deduce that:
    $$
    \prob\left[ i\left( t\right) = T_1 \Big| t\left( \epsilon\right)\right] \ge 1 - \frac{|\children\left( T\right)| - 1}{2}\textrm{erfc}\left( 4C\right)
    $$
    and thus:
    $$
    \prob\left[ i\left( t\right) = T_1\right] \ge \left[ 1 - \frac{|\children\left( T\right)| - 1}{2}\textrm{erfc}\left( 4C\right) \right]\sum_{1 \le \tau \le t}\prob\left[ t\left( \epsilon\right) = \tau\right]
    $$
    Since $\sum_{\tau \ge 1}\prob\left[ t\left( \epsilon\right) = \tau\right] = 1$ (because $\prob\left[ t\left( \epsilon\right) = \infty\right] = 0$), by taking the limit, we get:
    $$
    \lim_{t \rightarrow \infty}\prob\left[ i\left( t\right) = T_1\right] \ge 1 - \frac{|\children\left( T\right)| - 1}{2}\textrm{erfc}\left( 4C\right)
    $$
    Since this is satisfied for all $C > 0$ and $\lim_{C \rightarrow \infty}\textrm{erfc}\left( 4C\right) = 0$, we deduce that:
    $$
    \lim_{t \rightarrow \infty}\prob\left[ i\left( t\right) = T_1\right] = \lim_{t \rightarrow\infty}\pi_t\left( T_1 | T\right) = 1
    $$
\end{proof}

\begin{lemma}
    \label{lemma:conv-general}
    Let time $t$ denote the number of times that the parent of $T$ has been visited. Under the same assumptions as Lemma \ref{lemma:proba-opt}, $T$ satisfies:
    $$
    \mu_T \xlongrightarrow[t \rightarrow \infty]{\textrm{a.s}} \vf^{\pi^*}\left( T\right), \left( \sigma_T\right)^2 \xlongrightarrow[t \rightarrow \infty]{\textrm{a.s}} 0
    $$
\end{lemma}

\begin{proof}[Proof of Lemma \ref{lemma:conv-general}]
    By Induction on $n = |\children\left( T\right)|$, if $n = 2$ we have:
    \begin{align*}
    \mu_T &= \mu_{T_1}\Phi\left( \alpha\right) + \mu_{T_2}\Phi\left( -\alpha\right) + \phi\left( \alpha\right)\sigma_m\\
    \left( \sigma_T\right)^2 &= \left[ \left( \mu_{T_1}\right)^2 + \left( \sigma_{T_1}\right)^2\right]\Phi\left( \alpha\right) + \left[ \left( \mu_{T_2}\right)^2 + \left( \sigma_{T_2}\right)^2\right]\Phi\left( -\alpha\right) + \left( \mu_{T_1} + \mu_{T_2}\right)\sigma_m\phi\left( \alpha\right) - \left( \mu_T\right)^2
    \end{align*}
    Where $\sigma_m = \left(\sigma_{T_1}\right)^2 + \left(\sigma_{T_2}\right)^2, \alpha = \frac{\mu_1 - \mu_2}{\sigma_m}$, and $\phi$ and $\Phi$ are respectively the probability density function and the cumulative distribution function of $\mathcal{N}\left( 0, 1\right)$. Using Lemma \ref{lemma:visits-conv}, we have $N_{T_1}\left( t\right) \xlongrightarrow[t \rightarrow \infty]{\textrm{a.s}} \infty$ and $N_{T_2}\left( t\right) \xlongrightarrow[t \rightarrow \infty]{\textrm{a.s}} \infty$, therefore $\mu_{T_1} \xlongrightarrow[t \rightarrow \infty]{\textrm{a.s}} \vf^{\pi^*}\left( T_1\right), \sigma_{T_1} \xlongrightarrow[t \rightarrow \infty]{\textrm{a.s}} 0$ and $\mu_{T_2} \xlongrightarrow[t \rightarrow \infty]{\textrm{a.s}} \vf^{\pi^*}\left( T_2\right), \sigma_{T_2} \xlongrightarrow[t \rightarrow \infty]{\textrm{a.s}} 0$, which yields $\sigma_m \xlongrightarrow[t \rightarrow \infty]{\textrm{a.s}} 0$ and $\alpha \xlongrightarrow[t \rightarrow \infty]{\textrm{a.s}} \infty$, using these results with the formulas above, we deduce that:
    $$
    \mu_T \xlongrightarrow[t \rightarrow \infty]{\textrm{a.s}} \vf^{\pi^*}\left( T_1\right) = \max\{ \vf^{\pi^*}\left( T_1\right), \vf^{\pi^*}\left( T_2\right)\}, \; \sigma_T \xlongrightarrow[t \rightarrow \infty]{\textrm{a.s}} 0
    $$
    Now suppose the result holds true for some $n \ge 2$ and now consider $\children\left( T\right) = \{ T_1, \ldots, T_{n+1}\}$. We define $\theta_{T'}\left( t\right) \sim \mathcal{N}\left( \mu_{T'}, \left( \sigma_{T'}\right)^2\right)$ which approximates recursively the maximum for $\{ T_2, \ldots, T_{n+1}\}$. By the Induction hypothesis, we have:
    $$
    \mu_{T'} \xlongrightarrow[t \rightarrow \infty]{\textrm{a.s}} \max\{ \vf^{\pi^*}\left( T_2\right), \ldots, \vf^{\pi^*}\left( T_{n + 1}\right)\}, \; \sigma_{T'} \xlongrightarrow[t \rightarrow \infty]{\textrm{a.s}} 0
    $$
    Now we have $\theta_{T}\left( t\right) \sim \mathcal{N}\left( \mu_{T}, \left( \sigma_{T}\right)^2\right)$ approximating the maximum for $\{ T_1, T'\}$. By the Induction hypothesis again, we have:
    $$
    \mu_T \xlongrightarrow[t \rightarrow \infty]{\textrm{a.s}} \vf^{\pi^*}\left( T_1\right)=\max\{ \vf^{\pi^*}\left( T_1\right), \vf^{\pi^*}\left( T'\right)\}, \; \sigma_T \xlongrightarrow[t \rightarrow \infty]{\textrm{a.s}} 0
    $$
    This concludes the Induction proof.
\end{proof}

\begin{proof}[Proof of Theorem \ref{thm:main}]

By backward induction starting from the Search Leaves and going up to the Root Search Node, using Lemmas \ref{lemma:conv-search-leaves}, \ref{lemma:visits-conv}, \ref{lemma:conv-general} and Corollary \ref{cor:visits-conv}, it is straightforward to deduce the main result.
    
\end{proof}

\begin{lemma}
\label{lemma:std-lower-bound}
Let $T$ be a Search Leaf, with $t$ the number of visits of its parent and $N_T\left( t\right)$ the number of visits of $T$ up to time $t$. Then we have:
\begin{equation*}
    \forall{C} > |\leaves\left( T\right)|: N_{T}\left( t\right)m \le \frac{1}{2}\left( \frac{C}{|\leaves\left( T\right)|}\right)^{\frac{1}{2}} \implies \left( \sigma_T\right)^2 \ge \frac{1}{C}
\end{equation*}

\end{lemma}

\begin{proof}[Proof of Lemma \ref{lemma:std-lower-bound}]

We have: 
$$\left( \sigma_{T}\right)^2 = \sum_l \left( \hat{p}\left( l\right)\right)^2 \left( \sigma_{T, l}\right)^2; \left( \sigma_{T, l}\right)^2 = \frac{\alpha_{T, l}\beta_{T, l}}{\left( \alpha_{T, l} + \beta_{T, l}\right)^2\left( 1 + \alpha_{T, l} + \beta_{T, l}\right)}$$
Let $n_{T, l} = \sum_{i=2}^t\ind\{ X_i \in l\}$ the number of samples effectively observed in leaf $l$ of $T$ up to the $t^{\textrm{th}}$ visit of the parent of $T$. Then from Equation \eqref{eq:alpha} and \eqref{eq:beta}, we have the following $\alpha_{T, l} + \beta_{T, l} = 2+n_{T, l}$ and $\alpha_{T, l}\beta_{T, l} = \alpha_{T, l}\left( 2 + n_{T, l} - \alpha_{T, l}\right)$ knowing that $1 \le \alpha_{T, l} \le 1 + n_{T, l}$.\\
Consider the function $f : x \mapsto x\left( 2 + n_{T, l} - x\right)$ defined on $\left[ 1, 1 + n_{T,l}\right]$, $f$ has a minimum $f\left( 1\right) = f\left( 1 + n_{T, l}\right) = 1 + n_{T, l}$, and therefore:
\begin{equation*}
    \left( \sigma_{T, l}\right)^2 \ge \frac{1 + n_{T, l}}{\left( 2 + n_{T, l}\right)^2\left( 3 + n_{T, l}\right)}
\end{equation*}
Now consider the function $g : x \mapsto \frac{1+x}{\left( 2+x\right)^2\left( 3+x\right)}$ defined for $x \ge 0$, $g$ is differentiable on $\mathbb{R}^*_+$ and $g'\left( x\right) = \frac{-2 -6x -2x^2}{\left( 2+x\right)^3\left( 3+x\right)^2} < 0$, hence $g$ is decreasing. Therefore, for $\mathcal{C} > n_{T, l}$ we get:
$$
\left( \sigma_{T, l}\right)^2 \ge g\left( n_{T, l}\right) \ge g\left( \mathcal{C}\right)
$$
Furthermore, the total number of observed samples in $T$ is obviously larger than the number of observed samples in leaf $l \in \leaves\left( l\right)$, thus $n_{T, l} \le N_T\left( t\right) m \le \frac{1}{2}\sqrt{\frac{C}{|\leaves\left( T\right)|}}$. Let us choose $\mathcal{C} > 5$, this leads to:
\begin{align*}
    \left( \sigma_T\right)^2 &\ge g\left( \mathcal{C}\right)\sum_{l \in \leaves\left( T\right)}\hat{p}\left( l\right)\\
    \left( \sigma_T\right)^2 &\ge \frac{1}{|\leaves\left( T\right)|}\frac{1+\mathcal{C}}{\left( 2+\mathcal{C}\right)^2\left( 3+\mathcal{C}\right)}\\
    &\ge \frac{1}{|\leaves\left( T\right)|}\frac{\mathcal{C}}{\left( \sqrt{2}\mathcal{C}\right)^2\left( 2\mathcal{C}\right)}\\
    &\ge \frac{1}{4^ |\leaves\left( T\right)|\mathcal{C}^{2}}
\end{align*}
The second inequality comes from the fact that the uniform distribution minimises the collision probability.

By taking $\mathcal{C} = \frac{1}{2}\left( \frac{C}{|\leaves\left( T\right)|}\right)^{\frac{1}{2}}$, we deduce the result of the Lemma.
    
\end{proof}

\begin{proof}[Proof of Theorem \ref{thm:visits-proba}]

Let us first consider the case with $n=2$, i.e $\children\left( T\right) = \{ T_1, T_2\}$ and then we will generalise the result for an arbitrary $n \ge 2$. We have $M_2 = 1$, thus we want to show that:
\begin{align*}
    \prob\left[ N_{T_2}\left( t\right)m \le \frac{\log t}{4|\leaves\left( T_2\right)|}\right]
    \le \exp\left[ -\frac{2}{t}\left( \frac{t^{3/4}}{\sqrt{\pi}\left( \sqrt{\frac{\log t}{4}} + \sqrt{\frac{\log t}{4}+2}\right)} - \frac{\log t}{4m|\leaves\left( T_2\right)|}\right)^2\right]
\end{align*}
The result will be valid for $T_1$ as well without loss of generality.

At time $t$, child $T_2$ is chosen if $\theta_{T_2} \ge \theta_{T_1}$, which motivates us to study $\prob\left[ \theta_{T_2} \ge \theta_{T_1}\right]$:\\
We know that $\theta_{T_2} - \theta_{T_1} \sim \mathcal{N}\left( \mu_{T_2} - \mu_{T_1}, \left( \sigma_{T_2}\right)^2 + \left( \sigma_{T_1}\right)^2\right)$, hence:
\begin{align*}
    \prob\left[ \theta_{T_2} \ge \theta_{T_1}\right] &= \prob\left[ \theta_{T_2} \ge \theta_{T_1} \Big| \mu_{T_2} < \mu_{T_1}\right]\prob\left[ \mu_{T_2} < \mu_{T_1}\right] + \prob\left[ \theta_{T_2} \ge \theta_{T_1} \Big| \mu_{T_2} \ge \mu_{T_1}\right]\prob\left[ \mu_{T_2} \ge \mu_{T_1}\right]
\end{align*}
Since $\prob\left[ \theta_{T_2} \ge \theta_{T_1} \Big| \mu_{T_2} < \mu_{T_1}\right] \le \prob\left[ \theta_{T_2} \ge \theta_{T_1} \Big| \mu_{T_2} \ge \mu_{T_1}\right]$, we have:
\begin{align*}
    \prob\left[ \theta_{T_2} \ge \theta_{T_1}\right] &\ge \prob\left[ \theta_{T_2} \ge \theta_{T_1} \Big| \mu_{T_2} < \mu_{T_1}\right]\\
    &\ge \frac{1}{\sqrt{\pi}}\bigintsss_{\frac{\mu_{T_1} - \mu_{T_2}}{\sqrt{2\left[ \left( \sigma_{T_1}\right)^2 + \left( \sigma_{T_2}\right)^2\right]}}}^{\infty} e^{-u^2}du\\
    &\ge \frac{1}{\sqrt{\pi}}\bigintsss_{\frac{\mu_{T_1} - \mu_{T_2}}{\sqrt{2\left( \sigma_{T_2}\right)^2}}}^{\infty} e^{-u^2}du
\end{align*}
In what follows, we consider $N_{T_2}\left( t\right)m \le \frac{C}{2|\leaves\left( T\right)|}$, and we will define $C$ as a function of $t$ later. According to Lemma $\ref{lemma:std-lower-bound}$ we have $\left(\sigma_{T_2}\right)^2 \ge \frac{|\mu_{T_1} - \mu_{T_2}|}{C}$, thus $\frac{\mu_{T_1} - \mu_{T_2}}{\sqrt{2\left( \sigma_{T_2}\right)^2}} \le \sqrt{\frac{C\left( \mu_{T_1} - \mu_{T_2}\right)}{2}} \le \sqrt{\frac{C}{2}}$ since by definition, all the means $\mu_{T_i} \in \left[ 0, 1\right]$. This leads to:
\begin{align*}
    \prob\left[ \theta_{T_2} \ge \theta_{T_1} \Big| \left( \sigma_{T_2}\right)^2 \ge \frac{|\mu_{T_1} - \mu_{T_2}|}{C} \right] &\ge \frac{1}{\sqrt{\pi}}\bigintsss_{\sqrt{\frac{C}{2}}}^{\infty} e^{-u^2}du\\
    &\ge \frac{1}{2}\textrm{erfc}\left( \sqrt{\frac{C}{2}}\right)
\end{align*}
$\textrm{erfc}\left( .\right)$ denotes the complementary error function. Using the lower bound in \citep{kschischang2017complementary}, we deduce that:
\begin{equation}
    \label{ineq:ii}
    \prob\left[ \theta_{T_2} \ge \theta_{T_1} \Big| \left( \sigma_{T_2}\right)^2 \ge \frac{|\mu_{T_1} - \mu_{T_2}|}{C} \right] \ge \frac{1}{2}\textrm{erfc}\left( \sqrt{\frac{C}{2}}\right) \ge \frac{\exp\left( \frac{-C}{2}\right)}{\sqrt{\pi}\left( \sqrt{\frac{C}{2}} + \sqrt{\frac{C}{2} + 2}\right)}
\end{equation}
Since $\forall{t'} \le t: N_{T_2}\left( t'\right) \le N_{T_2}\left( t\right)$, we have $N_{T_2}\left( t\right)m \le \frac{C}{2|\leaves\left( T\right)|} \implies N_{T_2}\left( t'\right)m \le \frac{C}{2|\leaves\left( T\right)|} \implies \left( \sigma_{T_2}\right)^2 \ge \frac{\Big|\mu_{T_1} - \mu_{T_2}\Big|}{C}$, and Inequality \eqref{ineq:ii} holds for all $1 \le t' \le t$. Hence we write:
\begin{align*}
    \prob\left[ N_{T_2}\left( t\right)m \le \frac{C}{2|\leaves\left( T_2\right)|}\right] &= \prob\left[ N_{T_2}\left( t\right)m \le \frac{C}{2|\leaves\left( T_2\right)|}, \left( \sigma_{T_2}\right)^2 \ge \frac{|\mu_{T_2} - \mu_{T_1}|}{C}\right]\\
    &= \prob\left[ N_{T_2}\left( t\right)m \le \frac{C}{2|\leaves\left( T_2\right)|}, \forall{1 \le t' \le t}: \left( \sigma_{T_2}\right)^2 \ge \frac{\Big|\mu_{T_2} - \mu_{T_1}\Big|}{C}\right]\\
    &\le \prob\left[ N_{T_2}\left( t\right)m \le \frac{C}{2|\leaves\left( T_2\right)|} \Bigg| \forall{1 \le t' \le t}: \left( \sigma_{T_2}\right)^2 \ge \frac{\Big|\mu_{T_2} - \mu_{T_1}\Big|}{C}\right]
\end{align*}
Given the event $\Bigg\{ \forall{1 \le t' \le t}: \left( \sigma_{T_2}\right)^2 \ge \frac{\Big|\mu_{T_2} - \mu_{T_1}\Big|}{C} \Bigg\}$, at each time $1 \le t' \le t: T_2$ is chosen with probability $\prob\left[ \theta_{T_2} \ge \theta_{T_1}\Bigg| \left( \sigma_{T_2}\right)^2 \ge \frac{\Big|\mu_{T_2} - \mu_{T_1}\Big|}{C}\right]$.\\
Let $\mathcal{C} = \frac{\exp\left( \frac{-C}{2}\right)}{\sqrt{\pi}\left( \sqrt{\frac{C}{2}} + \sqrt{\frac{C}{2} + 2}\right)}$ and define the i.i.d. random variables $Z_1, \ldots, Z_t$ such that $\forall{i}: Z_i \sim \textrm{Bernoulli}\left( \mathcal{C}\right)$.\\
Inequality \eqref{ineq:ii} and Lemma \ref{lemma:std-lower-bound} lead to:
\begin{align*}
    \prob\left[ N_{T_2}\left( t\right)m \le \frac{C}{2|\leaves\left( T_2\right)|}\right] &= \prob\left[ N_{T_2}\left( t\right)m \le \frac{C}{2|\leaves\left( T_2\right)|}, \forall{1 \le t' \le t}: \left( \sigma_{T_2}\right)^2 \ge \frac{\Big|\mu_{T_2} - \mu_{T_1}\Big|}{C}\right] \\
    &\le \prob\left[ N_{T_2}\left( t\right)m \le \frac{C}{2|\leaves\left( T_2\right)|} \Bigg| \forall{1 \le t' \le t}: \left( \sigma_{T_2}\right)^2 \ge \frac{\Big|\mu_{T_2} - \mu_{T_1}\Big|}{C}\right]\\
    & \le \prob\left[ \sum_{i=1}^t Z_i \le \frac{C}{2m|\leaves\left( T_2\right)|}\right]
\end{align*}
Using Hoeffding's inequality, for $\epsilon>0$ we have $\prob\left[ \sum_{i=1}^t Z_i - t\mathcal{C} \le -\epsilon\right] \le \exp\left( -\frac{2\epsilon^2}{t}\right)$.\\
Thus, by setting $\epsilon = t\mathcal{C} - \frac{C}{2m|\leaves\left( T_2\right)|} > 0$, we deduce:
\begin{align*}
    \prob\left[ N_{T_2}\left( t\right)m \le \frac{C}{2|\leaves\left( T_2\right)|}\right] &\le \prob\left[ \sum_{i=1}^t Z_i \le \frac{C}{2m|\leaves\left( T_2\right)|}\right]\\
    &\le \exp\left( -\frac{2\left( t\mathcal{C} - \frac{C}{2m|\leaves\left( T_2\right)|}\right)^2}{t}\right)
\end{align*}
Now let us find an adequate expression of $C$ as a function of $t$, hence we will write $\mathcal{C}\left( t\right), C\left( t\right)$.\\
We recall that $\epsilon = t\mathcal{C} - \frac{C\left( t\right)}{2m|\leaves\left( T_2\right)|} > 0$, thus a first condition is to have $C\left( t\right) < 2m|\leaves\left( T_2\right)|t\mathcal{C}$, which means that $C\left( t\right)$ has to be sublinear.\\
Recall that $\mathcal{C}\left( t\right) = \frac{\exp\left( \frac{-C\left( t\right)}{2}\right)}{\sqrt{\pi}\left( \sqrt{\frac{C\left( t\right)}{2}} + \sqrt{\frac{C\left( t\right)}{2} + 2}\right)}$, thus:
\begin{align*}
    C\left( t\right) &< 2m|\leaves\left( T_2\right)|t\frac{\exp\left( \frac{-C\left( t\right)}{2}\right)}{\sqrt{\pi}\left( \sqrt{\frac{C\left( t\right)}{2}} + \sqrt{\frac{C\left( t\right)}{2} + 2}\right)}\\
    \frac{\sqrt{\pi}}{2m|\leaves\left( T_2\right)|}C\left( t\right)\left( \sqrt{\frac{C\left( t\right)}{2}} + \sqrt{\frac{C\left( t\right)}{2} + 2}\right) &< t\exp\left( -\frac{C\left( t\right)}{2}\right)
\end{align*}
For any $\alpha > 0$, we cannot have $C\left( t\right) = t^\alpha$ because the RHS would converge to $0$ as $t \rightarrow \infty$ while the LHS would diverge to $\infty$. Hence, we consider $C\left( t\right) = a\log t$ for some $a>0$.
\begin{align*}
    \frac{\sqrt{\pi}}{2m|\leaves\left( T_2\right)|}a\log\left( t\right)\left( \sqrt{\frac{a\log t}{2}} + \sqrt{\frac{a\log t}{2} + 2}\right) &< t\exp\left( -\frac{a\log t}{2}\right)\\
    &< t^{1 - \frac{a}{2}}
\end{align*}
Thus we must have $0 < a < 2$.
\begin{align*}
    \prob\left[ N_{T_2}\left( t\right)m \le \frac{a\log t}{2|\leaves\left( T_2\right)|}\right] &\le \exp\left( -\frac{2\left( \frac{t^{1 - \frac{a}{2}}}{\sqrt{\pi}\left( \sqrt{\frac{a\log t}{2}} + \sqrt{\frac{a\log t}{2} + 2}\right)} - \frac{a\log t}{2m|\leaves\left( T_2\right)|}\right)^2}{t}\right)
\end{align*}
Since $\frac{2}{t}\left( \frac{t^{1 - \frac{a}{2}}}{\sqrt{\pi}\left( \sqrt{\frac{a\log t}{2}} + \sqrt{\frac{a\log t}{2} + 2}\right)} - \frac{a\log t}{2m|\leaves\left( T_2\right)|}\right)^2 = \mathcal{O}\left( t^{1-a}\right)$, we must have $0 < a < 1$; by taking $a=\frac{1}{2}$, we get:
\begin{align*}
    \prob\left[ N_{T_2}\left( t\right)m \le \frac{\log t}{4|\leaves\left( T_2\right)|}\right]
    \le \exp\left[ -\frac{2}{t}\left( \frac{t^{3/4}}{\sqrt{\pi}\left( \sqrt{\frac{\log t}{4}} + \sqrt{\frac{\log t}{4}+2}\right)} - \frac{\log t}{4m|\leaves\left( T_2\right)|}\right)^2\right]
\end{align*}
Following the exact same steps, we can show that if $|\mu_{T_2} - \mu_{T_1}| \le M$ where $M > 0$ some constant, we would have:
\begin{align*}
    \prob\left[ N_{T_2}\left( t\right)m \le \frac{\log t}{4|\leaves\left( T_2\right)|M}\right]
    \le \exp\left[ -\frac{2}{t}\left( \frac{t^{3/4}}{\sqrt{\pi}\left( \sqrt{\frac{\log t}{4}} + \sqrt{\frac{\log t}{4}+2}\right)} - \frac{\log t}{4m|\leaves\left( T_2\right)|M^2}\right)^2\right]
\end{align*}
This result constitutes our induction hypothesis for the following generalisation.

Let us now address the setting with $\children\left( T\right) = \{ T_1, \ldots, T_n\}$ where $n \ge 2$, and let $i \in \{ 1, \ldots, n\}$. Our idea is to transform this problem into a problem with two children and use the induction hypothesis.\\
We consider a new child $T'$ with parameters $\theta_{T'} \sim \mathcal{N}\left( \mu_{T'}, \left( \sigma_{T'}\right)^2\right)$ such that:
\begin{align*}
    \mu_{T'} &= \max_{j \neq i}\{ \mu_{T_j}\} + f_n\left( t\right)\\
    \sigma_{T'} &= \max_{j \neq i}\{ \sigma_{T_j}\}
\end{align*}
$f_n\left( t\right)$ is a function that we will derive later on.\\
Consider the setting with the new set of children $\children'\left( T\right) = \{ T_i', T'\}$ where $\theta_{T_i} = \theta_{T_i'}$ and $|\leaves\left( T_i\right)| = |\leaves\left( T'_i\right)|$.\\
For any $C>0$, we want $\prob\left[ N_{T_i}\left( t\right)m \le C\right] \le \prob\left[ N_{T_i'}\left( t\right)m \le C\right]$, to achieve this, it suffices to have $\prob\left[ \theta_{T'} \ge \max_{j \neq i}\{ \theta_{T_j}\}\right] \ge \frac{1}{2}$ because it means that the probability of choosing $T_i'$ in the problem with $\children'\left( T\right)$ is lower than the probability of choosing $T_i$ in the problem with $\children\left( T\right)$, which leads to $\prob\left[ N_{T_i}\left( t\right)m \le C\right] \le \prob\left[ N_{T_i'}\left( t\right)m \le C\right]$. Using the union bound, we have:
\begin{align*}
    \prob\left[ \theta_{T'} \ge \max_{j \neq i}\{ \theta_{T_j}\}\right] \ge 1 - \sum_{j \neq i} \prob\left[ \theta_{T'} < \theta_{T_j}\right]
\end{align*}
Since $\forall{j}\neq i: \theta_{T'} - \theta_{T_j} \sim \mathcal{N}\left( \mu_{T'} - \mu_{T_j}, \left( \sigma_{T'}\right)^2 + \left( \sigma_{T_j}\right)^2\right)$, we have:
\begin{align*}
    \prob\left[ \theta_{T'} < \theta_{T_j}\right] &= \frac{1}{2}\textrm{erfc}\left( \frac{\max_{k \neq i}\{ \mu_{T_k}\} - \mu_{T_j} + f_n\left( t\right)}{\sqrt{2\left[ \left( \sigma_{T'}\right)^2 + \left( \sigma_{T_j}\right)^2\right]}}\right)\\
    &\le \frac{1}{2}\textrm{erfc}\left( \frac{f_n\left( t\right)}{2\sigma_{T'}}\right)
\end{align*}
Hence:
\begin{equation*}
    \prob\left[ \theta_{T'} \ge \max_{j \neq i}\{ \theta_{T_j}\}\right] \ge 1 - \frac{n-1}{2}\textrm{erfc}\left( \frac{f_n\left( t\right)}{2\sigma_{T'}}\right)
\end{equation*}
Thus, we want $f_n\left( t\right)$ satisfying $\textrm{erfc}\left( \frac{f_n\left( t\right)}{2\sigma_{T'}}\right) \le \frac{1}{n-1}$.\\
Take $f_n\left( t\right) = g_n\left( t\right)\sigma_{T'}$, hence it suffices to take $g_n\left( t\right) = 2\textrm{erfc}^{-1}\left( \frac{1}{n-1}\right)$ and thus 
$$f_n\left( t\right) = 2\sigma_{T'}\textrm{erfc}^{-1}\left( \frac{1}{n-1}\right)$$
In order to use the induction hypothesis, let us bound $\Big|\mu_{T'} - \mu_{T_j}\Big|$:
\begin{align*}
    \Big|\mu_{T'} - \mu_{T_j}\Big| &= \Big| \max_{j \neq i}\{ \mu_{T_j}\} + f_n\left( t\right) - \mu_{T_i}\left( t\right) \Big|\\
    &\le 1 + 2\sigma_{T'}\textrm{erfc}^{-1}\left( \frac{1}{n-1}\right)
\end{align*}
For any $j \neq i$, we have $\left( \sigma_{T_j}\right)^2 \le \sqrt{\frac{1}{12}}$, thus $\Big|\mu_{T'} - \mu_{T_j}\Big| \le 1 + \sqrt{\frac{2}{\sqrt{3}}}\textrm{erfc}^{-1}\left( \frac{1}{n-1}\right)$.\\
By defining $M_n = 1 + \sqrt{\frac{2}{\sqrt{3}}}\textrm{erfc}^{-1}\left( \frac{1}{n-1}\right)$, we use the induction hypothesis to deduce that:
\begin{align*}
    \prob\left[ N_{T_i'}\left( t\right)m \le \frac{\log t}{4|\leaves\left( T_i'\right)|M_n}\right]
    \le \exp\left[ -\frac{2}{t}\left( \frac{t^{3/4}}{\sqrt{\pi}\left( \sqrt{\frac{\log t}{4}} + \sqrt{\frac{\log t}{4}+2}\right)} - \frac{\log t}{4m|\leaves\left( T_1'\right)|M_n^2}\right)^2\right]
\end{align*}
Since $\prob\left[ N_{T_i}\left( t\right)m \le \frac{\log t}{4|\leaves\left( T_i\right)|M_n}\right] \le \prob\left[ N_{T_i'}\left( t\right)m \le \frac{\log t}{4|\leaves\left( T_i'\right)|M_n}\right]$, we deduce that:
\begin{align*}
    \prob\left[ N_{T_i}\left( t\right)m \le \frac{\log t}{4|\leaves\left( T_i\right)|M_n}\right]
    \le \exp\left[ -\frac{2}{t}\left( \frac{t^{3/4}}{\sqrt{\pi}\left( \sqrt{\frac{\log t}{4}} + \sqrt{\frac{\log t}{4}+2}\right)} - \frac{\log t}{4m|\leaves\left( T_i\right)|M_n^2}\right)^2\right]
\end{align*}
    
\end{proof}

\end{document}